\documentclass[conference]{IEEEtran}
\IEEEoverridecommandlockouts
\usepackage{algorithm}
\usepackage{algpseudocode}
\usepackage{amsmath,amssymb}
\usepackage{mathtools}
\usepackage{cite}
\usepackage{graphicx}
\usepackage{textcomp}
\usepackage{xcolor}
\usepackage{multirow}
\usepackage{array}
\usepackage{booktabs}
\usepackage{amsthm}
\usepackage{subcaption}
\usepackage[hidelinks]{hyperref}
\newtheorem{proposition}{Proposition}[section]

\newtheorem{theorem}[proposition]{Theorem}
\newtheorem{corollary}[proposition]{Corollary}
\newtheorem{assumption}{Assumption}[section]
\theoremstyle{definition}
\newtheorem{remark}{Remark}[section]

\def\BibTeX{{\rm B\kern-.05em{\sc i\kern-.025em b}\kern-.08em
    T\kern-.1667em\lower.7ex\hbox{E}\kern-.125emX}}
\begin{document}

\title{Three-Pronged Spectral Control for Federated Parameter Efficient Fine Tuning}

\author{Shiva Raj Pokhrel, Dipsan Bhattarai and Anwar Walid, \textit{Fellow, IEEE }
 \thanks{S.~R.~Pokhrel and D.~Bhattarai are with the School of IT, Deakin University, Australia. A.~Walid is with Columbia University, USA.  Corresponding author: S.~R.~Pokhrel; emails: \{shiva.pokhrel@deakin.edu.au, s225219805@deakin.edu.au, aie13@columbia.edu\}}}

\maketitle

\begin{abstract}
Federated parameter-efficient fine-tuning (PEFT) enables communication-efficient adaptation of large pretrained models on decentralized edge data, but it remains fragile under non-IID client heterogeneity. In low-rank adaptation (LoRA), different clients may learn locally useful but spectrally misaligned update subspaces, causing high-variance aggregation and poor global transfer. We propose \textit{TRISHUL}, a spectral-control framework for robust federated PEFT. TRISHUL follows the FL no-raw-data-sharing setting but does not itself provide formal privacy guarantees. TRISHUL uses shared frozen multi-head low-rank bases to obtain algebraically exact aggregation of compact core updates, applies nuclear-norm proximal shrinkage to suppress client-specific high-rank spectral components before upload, and allocates adaptation heads non-uniformly across layers using a concave water-filling budget rule derived from pretrained layer capacity. Because shrinkage is performed only on small core matrices, TRISHUL adds negligible computation and no extra per-round communication over the underlying multi-head PEFT protocol. Across vision and language benchmarks, including CIFAR-100, SVHN, 20~Newsgroups, MRQA, and GLUE with LLaMA3.2-1B, TRISHUL improves convergence, stability, and final performance over federated LoRA baselines, with greater gains under stronger heterogeneity.\footnote{See \href{https://github.com/Dipsan49/TRISHUL-Implementation} {Trishul Github} for implementation details.} %Spectral diagnostics show that TRISHUL reduces effective rank, spectral entropy, inter-client update variance, and subspace misalignment, supporting the proposed spectral-control interpretation.
\end{abstract}


\section{Introduction}

Large pretrained models have become the foundation of modern artificial intelligence, but adapting them efficiently in privacy-sensitive and resource-constrained edge environments remains a major challenge. In many practical deployments, data are generated locally on mobile devices, sensors, and edge platforms, where direct data sharing is undesirable or impossible due to privacy, bandwidth, and regulatory constraints. Federated learning (FL) addresses this challenge by enabling collaborative model training without centralizing raw data~\cite{mcmahan2017}. In parallel, parameter-efficient fine-tuning (PEFT), particularly low-rank adaptation (LoRA)~\cite{hu2021lora, liu2024efficient}, has emerged as a scalable approach for adapting large models using only a small number of trainable parameters. The combination of FL and PEFT~\cite{raje2025ravan} therefore offers a promising paradigm for the communication-efficient and privacy-preserving adaptation of foundation models at the edge.

% \begin{figure}
%     \centering
%     \includegraphics[width=0.49\linewidth]{Challenges.png}
%     \caption{Challenges in Federated PEFT}
%     \label{fig:placeholder}
% \end{figure}
Despite this promise, federated PEFT remains fundamentally challenging under realistic heterogeneous settings. Client data distributions are inherently non-IID, reflecting diverse user behaviors, environments, and tasks. This heterogeneity induces statistically inconsistent local updates, which can interfere destructively during aggregation and degrade global model performance. Classical federated optimization methods such as FedAvg~\cite{mcmahan2017}, FedProx~\cite{li2020fedprox} and SCAFFOLD~\cite{karimireddy2020scaffold} attempt to mitigate this issue through gradient correction, proximal regularization, or normalized aggregation. However, these approaches operate primarily at the level of optimization dynamics and do not explicitly control the \emph{structure} of parameter updates. The key problem is not only reducing trainable parameters, but ensuring that low-rank client updates remain geometrically compatible under aggregation.

\subsection{Motivation and Challenges}
Aforementioned limitations become critical in the context of LoRA. In PEFT, updates are constrained to low-dimensional subspaces, but under heterogeneous data, different clients may learn updates that occupy misaligned subspaces. Consequently, the aggregated update can exhibit \emph{spectral inconsistency}, where locally useful directions are globally incompatible. This phenomenon leads to high-variance updates and degraded transferability across clients. Recent methods such as multi-head LoRA~\cite{raje2025ravan} improve expressivity by introducing multiple subspaces per layer, but do not explicitly regulate the spectral geometry of these updates, leaving the aggregation instability unresolved.

These observations motivate a spectral-control formulation of federated PEFT. Rather than viewing the problem solely through the lens of optimization, we argue that heterogeneous federated fine-tuning should explicitly control update geometry under communication and resource constraints. From this perspective, two coupled failure modes arise: (i) \emph{aggregation interference}, caused by unstable and client-specific singular directions that increase the variance of aggregated updates, and (ii) \emph{capacity misallocation}, arising from uniform distribution of limited adaptation budget across layers despite significant variation in pretrained representational strength. Existing methods address these challenges only partially and in isolation.

To address the gap, we propose \textit{TRISHUL}, a three-pronged spectral-control framework for federated edge fine-tuning. It combines: (i) shared frozen multi-head bases for exact aggregation of compact core updates, (ii) nuclear-norm proximal shrinkage to suppress unstable client-specific spectral modes~\cite{recht2010,cai2010}, and (iii) concave water-filling to allocate a fixed adaptation budget across layers~\cite{boyd2004}. Together, these mechanisms improve aggregation fidelity and spectral stability without increasing per-round communication.

\subsection{Novelty and Contributions}

TRISHUL makes spectral inconsistency directly observable rather than inferring it from accuracy alone. We quantify client-update misalignment using principal-angle similarity, dominant singular-vector similarity, singular-value spectra, spectral entropy, effective rank, inter-client covariance, and aggregation variance. These diagnostics reveal whether non-IID clients learn incompatible low-rank subspaces and whether spectral shrinkage retains dominant shared modes while suppressing weak client-specific directions.

TRISHUL uses pretrained layer Frobenius norm as a stable, zero-cost allocation proxy and performs all spectral operations on compact $r\times r$ core matrices. It therefore preserves the communication cost of the underlying multi-head PEFT protocol while adding negligible local SVD overhead. The design also exposes a clear bias--variance tradeoff: insufficient shrinkage retains noisy modes, whereas excessive shrinkage removes useful task-specific structure.

The main contributions are:
\begin{itemize}
    \item We formulate heterogeneous federated PEFT as a \emph{spectral-control} problem, identifying update-subspace misalignment and inefficient layer-wise capacity allocation as coupled causes of aggregation instability.  We introduce client-side nuclear-norm proximal shrinkage on compact core matrices, enabling direct suppression of weak singular modes before aggregation while preserving the spectral structure of the corresponding full update.
    
    \item We combine shared-basis exact aggregation with a concave water-filling rule that allocates a fixed adaptation budget non-uniformly across layers without increasing per-round communication. We validate TRISHUL through accuracy, convergence, computational overhead, and direct spectral diagnostics, including subspace alignment, singular-value decay, effective rank, spectral entropy, and inter-client aggregation variance.
\end{itemize}

\begin{figure}[t!] 
    \centering
    \includegraphics[width=0.468\textwidth]{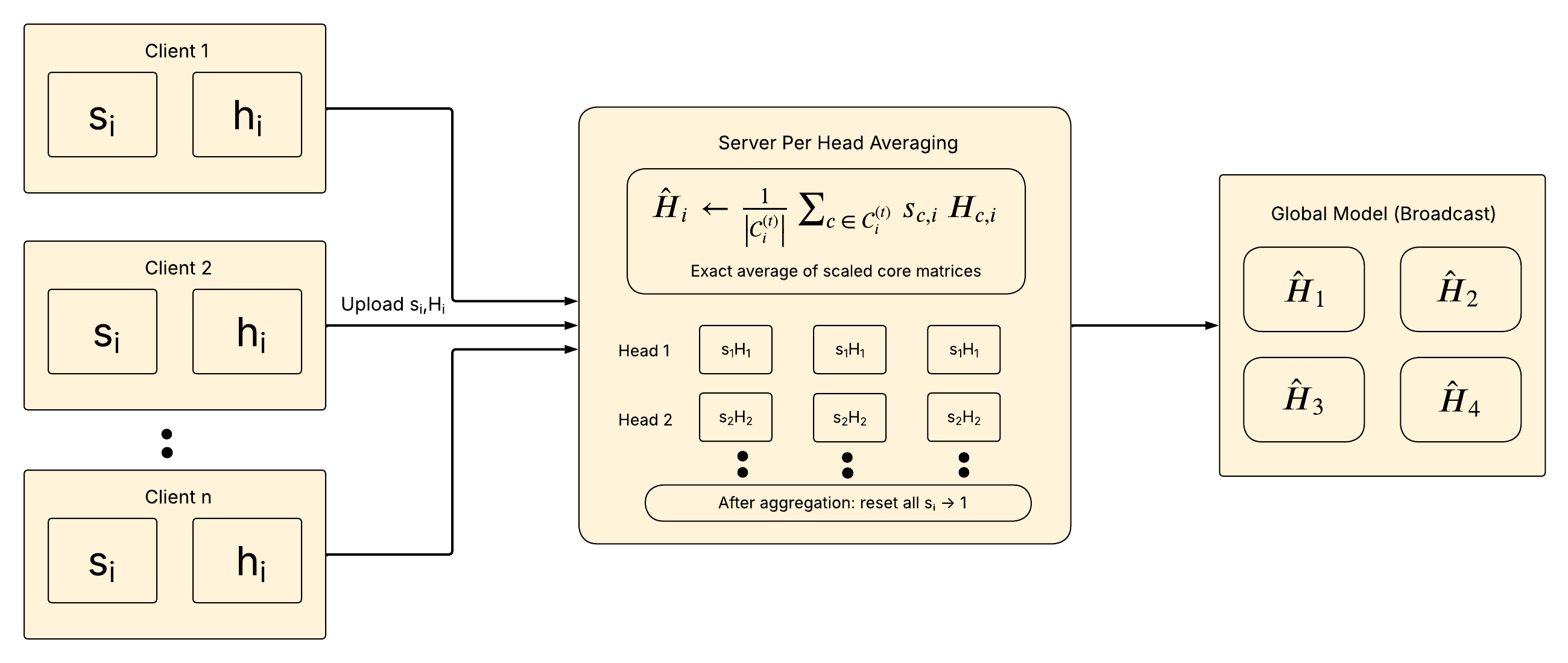}
    \caption{Multi-head low-rank aggregation (Prong 1). Each client uploads only the
    scaled core products $\{s_{c,i}\mathbf{H}_{c,i}\}$, which are averaged per head on
    the server and broadcast back as $\{\hat{\mathbf{H}}_1,\dots,\hat{\mathbf{H}}_4\}$.
    Because $\mathbf{B}_i$ and $\mathbf{A}_i$ are frozen and shared, this per-head
    average recovers the exact mean update in the ambient parameter space
    (Eq.~\eqref{eq:exact_agg}), eliminating the factor-wise aggregation bias
    $\mathbb{E}_c[\mathbf{B}_c\mathbf{A}_c]\neq\mathbb{E}_c[\mathbf{B}_c]\mathbb{E}_c[\mathbf{A}_c]$
    that corrupts standard LoRA under non-IID data. Server-side scalars are reset to
    $1$ after each round to prevent participation-frequency bias.}
    \label{fig:prong1}
\end{figure}

\section{Related Work}

\subsection{FL under Heterogeneity}

FL enables collaborative training across decentralized clients without sharing raw data~\cite{mcmahan2017}. A central challenge in FL is statistical heterogeneity, where client data distributions differ significantly, leading to biased and high-variance updates. A large body of work has focused on mitigating this issue through improved optimization strategies. FedProx~\cite{li2020fedprox} introduces proximal regularization to limit client drift, while SCAFFOLD~\cite{karimireddy2020scaffold} employs control variates to reduce client-gradient variance~\cite{liu2024efficient}.

While effective, these methods operate primarily at the level of gradient dynamics and do not explicitly regulate the \emph{structure} of parameter updates. In particular, they do not account for the geometric or spectral properties of updates, which become critical when adaptation is constrained to low-dimensional subspaces. As a result, even when optimization drift is reduced, aggregation can remain unstable due to structural misalignment of client updates.

\begin{table}[t]
\caption{TRISHUL relative to representative FL and PEFT methods.}
\label{tab:positioning}
\centering
\resizebox{\columnwidth}{!}{%
\begin{tabular}{lcccccc}
\toprule
Method & FL & PEFT & Exact agg. & Spectral shrink. & Layer alloc. & Edge heter. \\
\midrule
FedAvg~\cite{mcmahan2017} & \checkmark & -- & -- & -- & -- & Limited \\
FedProx~\cite{li2020fedprox} & \checkmark & -- & -- & -- & -- & Limited \\
SCAFFOLD~\cite{karimireddy2020scaffold} & \checkmark & -- & -- & -- & -- & Limited \\
FLoRA~\cite{wang2024flora} & \checkmark & \checkmark & Partial & -- & -- & Partial \\
FedEx-LoRA~\cite{singhal2024} & \checkmark & \checkmark & Partial & -- & -- & -- \\
Fed-SB~\cite{singhal2025} & \checkmark & \checkmark & -- & -- & -- & \checkmark \\
RAVAN~\cite{raje2025ravan} & \checkmark & \checkmark & \checkmark & -- & Limited & Partial \\
\textit{TRISHUL} & \checkmark & \checkmark & \checkmark & \checkmark & \checkmark & \checkmark \\
\bottomrule
\end{tabular}}
\end{table}
\subsection{PEFT in FL Settings}

PEFT methods, such as LoRA~\cite{hu2021lora}, enable the adaptation of large models by restricting updates to low-rank subspaces. This paradigm is particularly attractive in federated settings due to its reduced communication and memory footprint. Recent work extends PEFT to FL by adapting low-rank updates across clients. For example, FLoRA~\cite{wang2024flora} and related approaches~\cite{sun2024,qi2024fdlora,guo2025fedsa,cho2024,bai2024} address heterogeneity through personalized adaptation, selective aggregation, or heterogeneous rank allocation. FedEx-LoRA~\cite{singhal2024} improves aggregation by averaging low-rank updates directly, while Fed-SB~\cite{singhal2025} focuses on communication-efficient 
fine-tuning.\footnote{Exact agg. in this context (see Tab.~\ref{tab:positioning}) is a precise aggregation in ambient update space under shared bases/cores, not generic parameter averaging.}

Recent centralized PEFT methods such as AdaLoRA~\cite{zhang2023adalora},  DoRA~\cite{liu2024dora}, VeRA~\cite{kopiczko2024vera}, PiSSA~\cite{meng2024pissa}, and orthogonal LoRA variants such as  OLoRA~\cite{buyukakyuz2024olora} improve adaptation efficiency through adaptive rank allocation, weight decomposition, low-dimensional reparameterization, or improved initialization. These methods are complementary to TRISHUL. Their primary focus is centralized  parameterization quality or rank efficiency, whereas TRISHUL targets a federated failure mode: spectrally inconsistent client updates under non-IID data. In particular, TRISHUL  regulates the geometry of client updates before aggregation and allocates adaptation capacity across layers under a federated communication budget. More recently, RAVAN~\cite{raje2025ravan} introduces a multi-head low-rank parameterization to increase representational flexibility and enable exact aggregation of low-rank updates. Their approach largely treat low-rank updates as independent local adaptations and do not explicitly control the spectral structure. Under non-IID data, different clients may learn updates that span misaligned subspaces, leading to \emph{spectral inconsistency} and high-variance aggregation. Existing federated PEFT methods do not account for such limits.

\subsection{Spectral Regularization and LoRA}

Controlling the rank and spectral properties of model updates has been extensively studied in optimization and matrix recovery. Nuclear-norm regularization provides a convex surrogate for rank minimization and has been widely used in matrix completion and low-rank recovery~\cite{candes2009,recht2010}. Proximal methods based on \textit{singular value thresholding} (SVT) enable efficient optimization under nuclear-norm constraints~\cite{cai2010,beck2009}. 

In deep learning, the spectral properties of the weights and updates have been shown to influence generalization and stability~\cite{martin2021implicit}. However, these techniques are typically applied in centralized settings and have not been explicitly adapted to federated PEFT. In particular, previous work does not exploit spectral regularization to mitigate aggregation variance arising from heterogeneous client updates.

\subsection{Resource Allocation in Federated Setting}

Efficient allocation of limited computational and communication resources is a key concern in distributed learning systems. Prior work has explored adaptive communication strategies, model compression, and client selection to improve efficiency~\cite{boyd2004}. In federated PEFT, several methods consider heterogeneous resource constraints across clients~\cite{cho2024,bai2024}, including varying model capacity or partial updates.

However, the allocation of adaptation capacity across \emph{model layers} remains underexplored. Existing approaches typically assign uniform parameter budgets across layers, ignoring differences in pretrained representational strength. This can lead to inefficient use of limited adaptation capacity, particularly in deep transformer architectures, where layers contribute unevenly to downstream performance.

\subsection{Our Proposed Approach}

Our proposed TRISHUL differs from prior work by treating federated PEFT as a \emph{spectral-control} problem rather than only an optimization-drift or parameter-efficiency problem. The shared-basis multi-head structure is used as an aggregation substrate, while the primary methodological distinction is that TRISHUL regularizes the singular spectrum of each client update before upload and assigns adaptation capacity across layers under a fixed communication budget. This separates TRISHUL from RAVAN and similar PEFT approaches, which provides exact multi-head aggregation but does not explicitly shrink client-specific singular directions. To the best of our knowledge, TRISHUL is among the first federated PEFT frameworks to jointly control update geometry, aggregation variance, and layer-wise capacity allocation within one spectral framework.

\section{Design Details of Proposed TRISHUL}

We present \textit{TRISHUL}, a three-pronged spectral control framework for heterogeneous federated edge fine-tuning. TRISHUL is motivated by a central observation: in federated PEFT, performance degradation under non-IID data is driven not only by optimization drift, but also by the spectral inconsistency of client updates and the inefficient use of scarce adaptation capacity across layers. 

As shown in Fig.~\ref{fig:placeholder}, TRISHUL addresses these issues through exact multi-head low-rank aggregation, nuclear-norm proximal spectral shrinkage, and concave water-filling head allocation. The first ensures algebraically faithful aggregation, the second suppresses unstable client-specific singular directions before upload, and the third distributes a fixed adaptation budget toward layers with higher pretrained capacity. TRISHUL introduces no additional per-round communication beyond the underlying multi-head core-upload protocol.

\begin{figure}[h]
    \centering
    \includegraphics[width=0.85\linewidth]{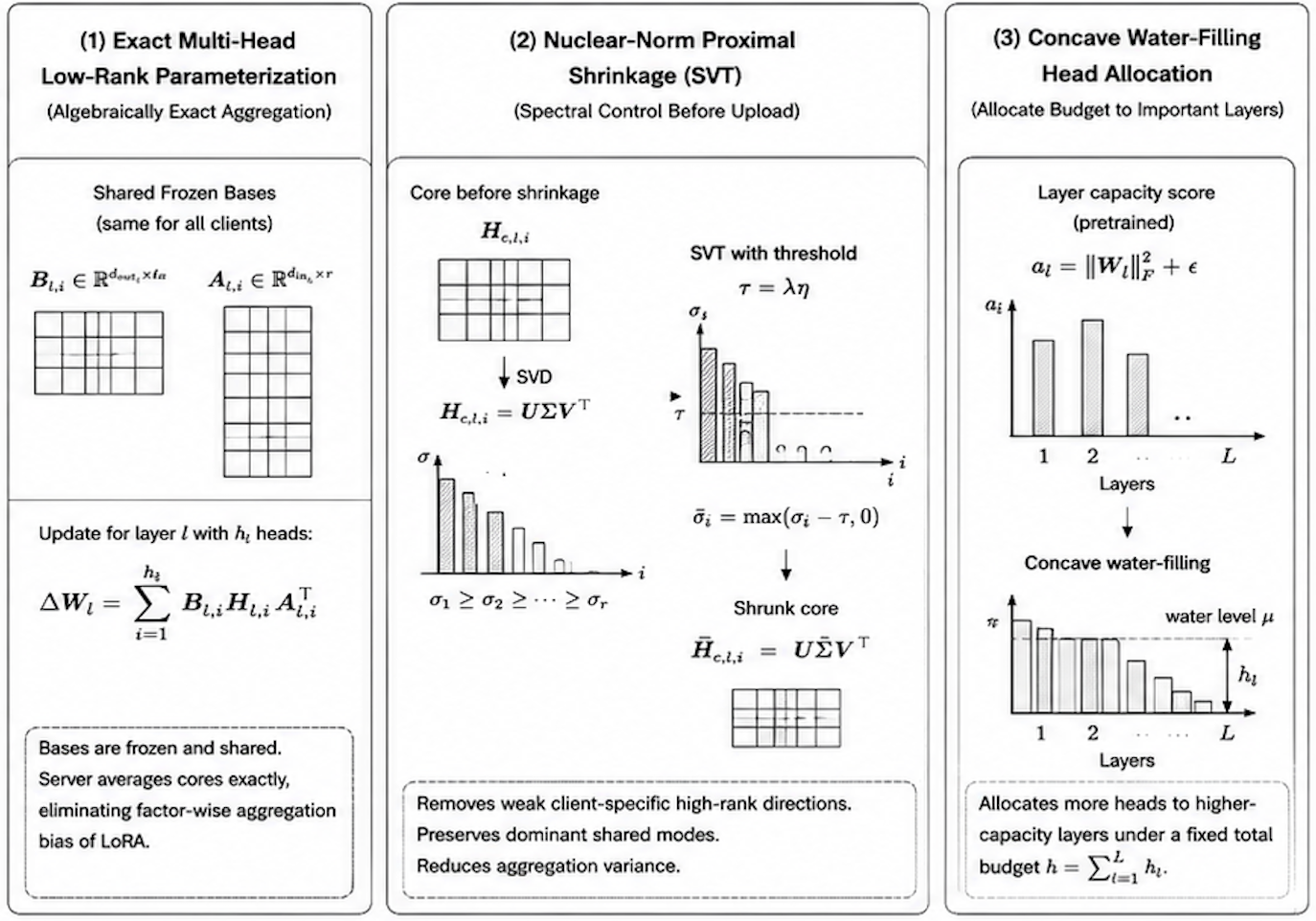}
    \caption{Three TRISHUL mechanisms applied during local client adaptation before server aggregation.}
    \label{fig:placeholder}
\end{figure}
\subsection{Problem Formulation and Spectral-Control Objective}
\label{subsec:problem}

Let $\mathcal{C}=\{1,\ldots,C\}$ denote the clients and let $\mathcal{D}_c$ be the local data distribution of client $c$. Federated fine-tuning minimizes
\begin{equation}
    F(\mathbf{W})
    =
    \sum_{c=1}^{C}p_c\,\mathcal{L}_c(\mathbf{W}),
    \qquad
    p_c=\frac{n_c}{\sum_{j=1}^{C}n_j},
    \label{eq:federated_objective}
\end{equation}
where $n_c$ is the local sample size and $\mathcal{L}_c$ is the empirical loss on client $c$. Non-IID heterogeneity means that $\mathcal{D}_c$ and $\mathcal{D}_{c'}$ may differ substantially, so locally optimized PEFT updates need not occupy compatible low-dimensional subspaces. Since $W_0$ and the bases are frozen, the actual trainable variables are $\theta={H_{l,i},s_{l,i}}$. For a pretrained layer $\mathbf{W}_l^0$, TRISHUL learns only an additive PEFT update
\begin{equation}
    \mathbf{W}_l
    =
    \mathbf{W}_l^0+\Delta\mathbf{W}_l,
    \qquad
    \Delta\mathbf{W}_l
    =
    \sum_{i=1}^{h_l}
    s_{l,i}\mathbf{B}_{l,i}\mathbf{H}_{l,i}\mathbf{A}_{l,i}.
    \label{eq:trishul_layer_update}
\end{equation}
The trainable communication budget (with state scalars folded and no added asymptotic cost) is
\begin{equation}
    \mathcal{B}
    =
    \sum_{l=1}^{L}h_l r^2,
    \label{eq:budget}
\end{equation}
because only the $r\times r$ core products are communicated during federated training. TRISHUL therefore seeks to minimize the federated loss while controlling three coupled quantities: the algebraic fidelity of aggregation, the spectral complexity of each client update, and the allocation of the fixed budget $\mathcal{B}$ across layers. 

\subsection{Multi-Head Low-Rank Parameterization}
\label{subsec:param}

Direct fine-tuning of all parameters of a large pretrained model on edge devices is generally infeasible. The memory footprint of full-parameter updates is too large, local optimization is too expensive, and repeated communication of dense weight updates is prohibitive under practical bandwidth constraints. PEFT addresses this difficulty by restricting adaptation to a low-dimensional update space~\cite{hu2021lora}. Among PEFT methods, LoRA has become a dominant approach due to its favorable memory and communication profile. However, its straightforward use in federated environments introduces an aggregation pathology that becomes increasingly severe under client heterogeneity. TRISHUL therefore builds on a multi-head low-rank parameterization that preserves PEFT efficiency while eliminating the factor-wise aggregation error of standard LoRA and creating a natural substrate for spectral regularization and adaptive capacity allocation.

\subsubsection{Background: LoRA and Its Federated Limitation}
\label{subsubsec:background}

LoRA adapts a pretrained weight matrix $\mathbf{W}\in\mathbb{R}^{d\times d}$ by adding a low-rank update of the form
\begin{equation}
    \mathbf{W} \leftarrow \mathbf{W} + \mathbf{B}\mathbf{A},
    \qquad
    \mathbf{B}\in\mathbb{R}^{d\times r},\;
    \mathbf{A}\in\mathbb{R}^{r\times d},\;
    r \ll d.
    \label{eq:lora}
\end{equation}
This reduces the number of trainable parameters from $d^2$ to $2dr$, which substantially lowers local memory consumption and communication cost. In centralized training, the two factors $\mathbf{B}$ and $\mathbf{A}$ are optimized jointly, so their product directly represents the learned update. In federated training, however, each client independently optimizes local copies $\mathbf{B}_c$ and $\mathbf{A}_c$, and the server must aggregate them across heterogeneous local data distributions.

A natural but problematic strategy is to average the factors separately and reconstruct the update as $\bar{\mathbf{B}}\bar{\mathbf{A}}$. Under non-IID data, this induces a systematic aggregation bias because the factors are statistically dependent on the local client distribution:
\begin{equation}
    \mathbb{E}_{c}\bigl[\mathbf{B}_{c}\mathbf{A}_{c}\bigr]
    \neq
    \mathbb{E}_{c}[\mathbf{B}_{c}]\,\mathbb{E}_{c}[\mathbf{A}_{c}].
    \label{eq:lora_bias}
\end{equation}
The discrepancy is governed by the cross-client covariance between $\mathbf{B}_{c}$ and $\mathbf{A}_{c}$ and generally grows as heterogeneity increases~\cite{raje2025ravan}. Consequently, factor-wise averaging underestimates the true mean update and introduces a distribution-dependent distortion into the global model. This limitation is particularly harmful at the edge, where local data are highly personalized and strongly non-IID. Although stacking-based aggregation strategies partially alleviate this problem~\cite{wang2024flora}, they increase complexity and still do not regulate the spectral structure of the resulting updates.

\subsubsection{Exact Multi-Head Adaptation}
\label{subsubsec:multihead}

To eliminate the aggregation bias in~\eqref{eq:lora_bias}, TRISHUL adopts a multi-head low-rank formulation inspired by recent federated PEFT designs~\cite{raje2025ravan}. For layer $l$, the model output is modified as
\begin{equation}
    \mathbf{W}\mathbf{x}
    \leftarrow
    \mathbf{W}\mathbf{x}
    +
    \sum_{i=1}^{h_l}
    s_i \,\mathbf{B}_i \mathbf{H}_i \mathbf{A}_i \mathbf{x},
    \label{eq:forward}
\end{equation}
where $h_l$ is the number of adaptation heads assigned to layer $l$, $\mathbf{B}_i\in\mathbb{R}^{d\times r}$ and $\mathbf{A}_i\in\mathbb{R}^{r\times d}$ are frozen basis matrices, $\mathbf{H}_i\in\mathbb{R}^{r\times r}$ is a trainable core matrix, and $s_i\in\mathbb{R}$ is a trainable scalar.

The frozen bases satisfy
\[
\mathbf{B}_i^\top\mathbf{B}_i=\mathbf{I}_r,
\qquad
\mathbf{A}_i\mathbf{A}_i^\top=\mathbf{I}_r,
\]
and are mutually orthogonal across heads:
\[
\mathbf{B}_i^\top\mathbf{B}_j=\mathbf{0},
\qquad
\mathbf{A}_i\mathbf{A}_j^\top=\mathbf{0},
\qquad i\neq j.
\]
These conditions ensure that each head spans a distinct low-dimensional subspace and that the collection of $h_l$ heads covers a rank-$h_l r$ adaptation space with maximal diversity for the chosen budget~\cite{raje2025ravan}. Since the bases are initialized once on the server, then frozen and shared identically across clients, they do not contribute to aggregation ambiguity.

The trainable core matrix $\mathbf{H}_i$ captures all task-specific adaptation inside the subspace defined by $(\mathbf{B}_i,\mathbf{A}_i)$. This reduces trainable parameters per head to only $r^2$, which is substantially smaller than directly adapting the ambient $d\times d$ matrix. The scalar $s_i$ provides a lightweight mechanism to modulate the contribution of each head according to local task relevance. To prevent participation-frequency bias across rounds, the server resets all global scalars to $1$ after aggregation.

\subsubsection{Initialization of Frozen Bases}
\label{subsubsec:init}

Because the frozen bases define the adaptation subspaces available throughout training, their initialization materially affects the expressive power and geometric diversity of the multi-head update family. TRISHUL considers three initialization strategies.

\textit{Gram--Schmidt orthogonalization} constructs maximally separated heads by iteratively projecting each new basis away from previously chosen directions. This yields exact inter-head orthogonality and is particularly effective in vision models, where broad spatial feature diversity benefits from highly separated adaptation subspaces.

\textit{Random normal initialization} draws entries of $\mathbf{B}_i$ and $\mathbf{A}_i$ from $\mathcal{N}(0,1)$ and then applies QR orthogonalization. This produces orthonormal bases with random orientations and a softer diversity constraint. For language models, whose pretrained representations are often anisotropic, this softer structure can better align with the geometry of the underlying representation manifold~\cite{martin2021implicit}.

\textit{Shared-subspace initialization} assigns multiple heads bases spanning the same underlying subspace, differing only by an invertible transformation. This reduces effective head diversity and serves mainly as a controlled ablation baseline.

In all cases, the core matrices are initialized as $\mathbf{H}_i=\mathbf{0}$ and the scalars as $s_i=1$, so the adaptation begins exactly from the pretrained model without any initial perturbation.

\subsubsection{Key Properties for Edge Federated Learning}

The multi-head formulation used by TRISHUL provides three properties that are especially important for federated edge deployment.

\paragraph{Exact aggregation.}
Because $\mathbf{B}_i$ and $\mathbf{A}_i$ are frozen and shared across all clients, aggregation becomes exact at the level of the trainable core products. Let $\mathcal{C}^{(t)}$ be the participating client set in round $t$. Then
\begin{align}
    &\frac{1}{|\mathcal{C}^{(t)}|}
    \sum_{c\in\mathcal{C}^{(t)}}
    \sum_{i=1}^{h_l}
    \mathbf{B}_i\bigl(s_{c,i}^{(t)}\mathbf{H}_{c,i}^{(t)}\bigr)\mathbf{A}_i
    \nonumber\\
    &=
    \sum_{i=1}^{h_l}
    \mathbf{B}_i
    \left(
    \frac{1}{|\mathcal{C}^{(t)}|}
    \sum_{c\in\mathcal{C}^{(t)}}
    s_{c,i}^{(t)}\mathbf{H}_{c,i}^{(t)}
    \right)
    \mathbf{A}_i.
    \label{eq:exact_agg}
\end{align}
Thus the server only needs to average the small matrices $s_{c,i}\mathbf{H}_{c,i}$, and the exact mean update in the ambient parameter space is recovered without approximation. This directly resolves the factor-wise aggregation bias of standard LoRA.

\paragraph{Handling computational heterogeneity}
Edge devices vary significantly in compute and memory capability. In the multi-head formulation, a constrained client may update only a subset of the assigned heads and freeze the remainder. Aggregation is then performed head-wise over only those clients that updated the corresponding head. This mechanism supports partial local adaptation without requiring protocol changes or additional coordination.

\paragraph{Communication efficiency}
Each client uploads only the products $\{s_{c,l,i}\mathbf{H}_{c,l,i}\}\in\mathbb{R}^{r\times r}$. The pretrained weights and frozen basis matrices are not communicated during normal training. Consequently, the per-round upload cost remains $\sum_{l=1}^{L} h_l r^2$ parameters, matching the underlying multi-head PEFT budget and preserving suitability for bandwidth-limited settings.

\begin{figure}[t!] 
    \centering
    \includegraphics[width=0.4\textwidth]{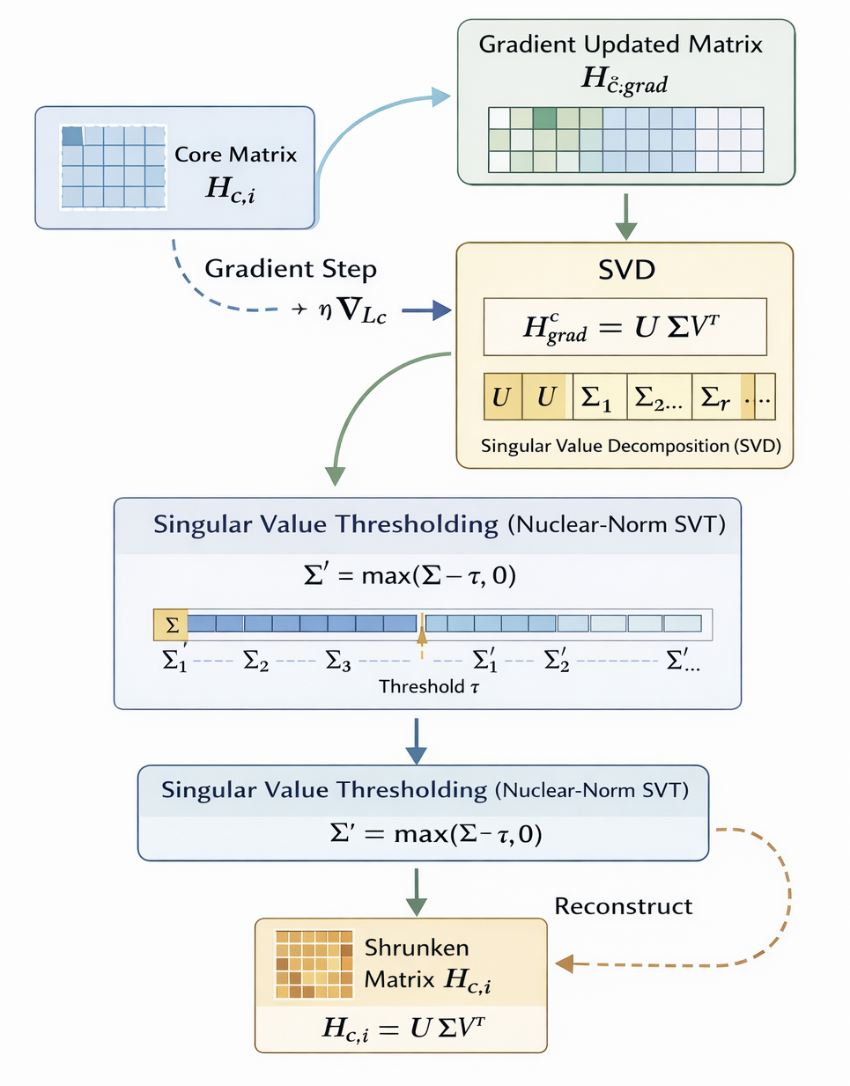}
    \caption{Proximal singular-value thresholding (Prong 2). Singular values below $\tau=\lambda\eta$ are removed, while dominant modes are retained with reduced magnitude. Because $\mathbf{B}_i$ and $\mathbf{A}_i$ are isometric, shrinkage on the $r\times r$ core corresponds to spectral regularization of the full update.}
    \label{fig:svt_prong2}
\end{figure}

\subsection{Client-Side Optimization with Nuclear-Norm}
\label{subsec:norm}

Although exact aggregation resolves the factorization bias of standard LoRA, it does not by itself prevent destructive interference among heterogeneous client updates. Under non-IID data, local fine-tuning tends to produce update directions that are highly specific to individual client distributions. When these directions are numerous and spectrally diffuse, their aggregation can have high variance even if the aggregation rule is exact. Standard drift-control methods, such as proximal penalties or control variates~\cite{li2020fedprox,karimireddy2020scaffold}, primarily constrain update magnitude or gradient dynamics. They do not explicitly control the spectral complexity of the learned update itself. TRISHUL addresses this failure mode by imposing direct spectral regularization on the trainable core matrices.

The central idea is to discourage unstable high-rank local adaptations while preserving dominant low-dimensional directions that are more likely to transfer across clients. To achieve this, TRISHUL penalizes the nuclear norm
\[
\|\cdot\|_*=\sum_j \sigma_j(\cdot),
\]
which is the tightest convex surrogate for rank minimization~\cite{recht2010,candes2009}. A key advantage of the multi-head parameterization is that regularizing the small core matrix is exactly equivalent to regularizing the full-dimensional update. Since $\mathbf{B}_i$ and $\mathbf{A}_i$ are orthonormal, they act as isometric embeddings and preserve singular values:
\begin{equation}
    \|\mathbf{B}_i\mathbf{H}_{c,i}\mathbf{A}_i\|_*
    =
    \|\mathbf{H}_{c,i}\|_*.
    \label{eq:norm_equiv}
\end{equation}
Therefore, imposing a nuclear-norm penalty on $\mathbf{H}_{c,i}$ directly constrains the spectral complexity of the full update in the ambient parameter space, while requiring only an $r\times r$ singular value decomposition.

For client $c$, the local optimization problem becomes
\begin{equation}
    \min_{\{\mathbf{H}_{c,i},s_{c,i}\}}
    \mathcal{L}_c\bigl(\{\mathbf{H}_{c,i},s_{c,i}\}\bigr)
    +
    \lambda\sum_{i=1}^{h_l}\|\mathbf{H}_{c,i}\|_*,
    \label{eq:local_obj}
\end{equation}
where $\mathcal{L}_c$ is the task loss on local data and $\lambda>0$ controls the strength of spectral regularization. Because the nuclear norm is non-smooth, TRISHUL employs proximal gradient descent. Each local step consists of a standard gradient update on the task loss followed by a proximal singular-value thresholding (SVT) step:
\begin{align}
    \mathbf{H}_{c,i}^{\mathrm{grad}}
    &=
    \mathbf{H}_{c,i}
    -
    \eta \nabla_{\mathbf{H}_{c,i}}\mathcal{L}_c,
    \label{eq:grad_step}
    \\
    \mathbf{H}_{c,i}
    &\leftarrow
    \operatorname{prox}_{\lambda\eta\|\cdot\|_*}
    \bigl(\mathbf{H}_{c,i}^{\mathrm{grad}}\bigr),
    \label{eq:prox_step}
    \\
    s_{c,i}
    &\leftarrow
    s_{c,i}
    -
    \eta \nabla_{s_{c,i}}\mathcal{L}_c.
    \label{eq:scalar_step}
\end{align}
If $\mathbf{X}=\mathbf{U}\operatorname{diag}(\sigma_j)\mathbf{V}^\top$, then the proximal operator is
\begin{equation}
    \operatorname{prox}_{\tau\|\cdot\|_*}(\mathbf{X})
    =
    \mathbf{U}\operatorname{diag}\bigl(\max(\sigma_j-\tau,0)\bigr)\mathbf{V}^\top,
    \qquad
    \tau=\lambda\eta.
    \label{eq:svt}
\end{equation}
Hence every singular value below the threshold $\tau$ is removed, while larger singular values are shrunk toward zero. This suppresses client-specific high-rank directions and retains only the dominant spectral modes.

The variance-reduction effect follows from the rank control induced by SVT. Let
\[
\widetilde{\mathbf{H}}_{c,i}
=
\operatorname{prox}_{\lambda\eta\|\cdot\|_*}
(\mathbf{H}^{\mathrm{grad}}_{c,i})
\]
denote the post-SVT core matrix and suppose $|s_{c,i}|\leq s_{\max}$ and
$\|\widetilde{\mathbf{H}}_{c,i}\|_2\leq M_i$ for all participating clients. If
\[
\rho_i
=
\bigl|\{j:\sigma_j(\mathbf{H}_{c,i}^{\mathrm{grad}})>\lambda\eta\}\bigr|
\]
is the maximum retained rank after thresholding, then
\begin{equation}
    \|\widetilde{\mathbf{H}}_{c,i}\|_F^2
    =
    \sum_{j=1}^{\rho_i}\widetilde{\sigma}_{j}^{\,2}
    \leq
    \rho_i M_i^2.
    \label{eq:rank_frobenius}
\end{equation}
For $m=|\mathcal{C}^{(t)}|$ independently sampled clients, the per-head aggregation variance satisfies
\begin{equation}
    \mathbb{E}
    \left\|
    \frac{1}{m}
    \sum_{c\in\mathcal{C}^{(t)}}
    \left(
    s_{c,i}\widetilde{\mathbf{H}}_{c,i}
    -
    \mathbb{E}[s_{c,i}\widetilde{\mathbf{H}}_{c,i}]
    \right)
    \right\|_F^2
    \leq
    \frac{s_{\max}^2\rho_i M_i^2}{m}.
    \label{eq:variance_bound}
\end{equation}
Thus, shrinking weak singular modes reduces the effective rank $\rho_i$ and tightens the aggregation-variance bound. This also exposes the bias--variance tradeoff: moderate shrinkage suppresses client-specific spectral noise, whereas overly aggressive shrinkage may remove task-relevant singular directions.
\begin{figure}[t!] 
    \centering
    \includegraphics[width=0.48\textwidth]{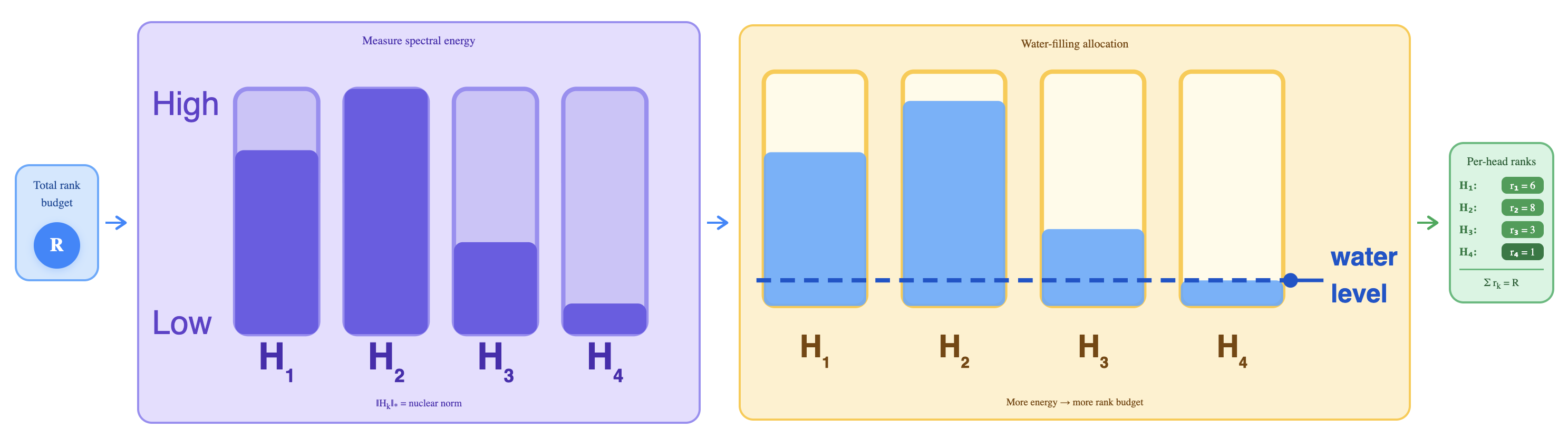}
    \caption{Concave water-filling head allocation (Prong 3). Each layer's
    pretrained capacity score $a_l=\|\mathbf{W}_l\|_F^2+\epsilon$ is measured and
    used to solve the diminishing-return allocation problem in
    Eq.~\eqref{eq:waterfill_opt}. High-capacity layers receive more adaptation
    heads, while the logarithmic utility prevents the full budget from collapsing
    onto a single layer. The resulting integer head counts are obtained by
    residual rounding of the continuous solution in Eq.~\eqref{eq:waterfill_sol}.}
    \label{fig:allocation_prong3}
\end{figure}

\subsection{Concave Water-Filling Head Allocation}
\label{sec:waterfill}

Exact aggregation and spectral regularization stabilize the uploaded updates, but they do not determine how a limited adaptation budget should be distributed across layers. Uniformly assigning the same number of heads to every layer is rarely optimal because pretrained layers differ in norm, abstraction level, and downstream sensitivity~\cite{dosovitskiy2021,raffel2020}. TRISHUL therefore uses a concave water-filling rule that allocates more heads to layers with larger pretrained capacity while preventing degenerate concentration of the entire budget in one layer.

Let $h=\sum_{l=1}^{L}h_l$ be the total head budget and define a positive layer-capacity score
\begin{equation}
    a_l
    =
    \|\mathbf{W}_l\|_F^2+\epsilon,
    \qquad
    \epsilon>0.
    \label{eq:layer_capacity}
\end{equation}
We allocate continuous heads by solving the strictly concave program
\begin{align}
    \max_{\{h_l\geq 0\}}
    \;&
    \sum_{l=1}^{L}
    a_l\log(1+h_l)
    \nonumber\\
    \text{s.t.}
    \;&
    \sum_{l=1}^{L}h_l=h.
    \label{eq:waterfill_opt}
\end{align}
The logarithmic utility models diminishing returns: layers with larger $a_l$ receive more capacity, but the marginal gain decreases as more heads are assigned. The KKT conditions give
\begin{equation}
    \frac{a_l}{1+h_l^*}-\mu\leq 0,
    \qquad
    h_l^*
    =
    \left[\frac{a_l}{\mu}-1\right]_+,
    \label{eq:waterfill_sol}
\end{equation}
where $[x]_+=\max(x,0)$ and the water level $\mu>0$ is chosen so that $\sum_l h_l^*=h$. This is the standard water-filling form: layers below the water level receive no additional heads, while high-capacity layers receive capacity proportional to their excess above the threshold.

Because head counts must be integer, TRISHUL uses the continuous solution only as a budget score. We set
\begin{equation}
    \widehat{h}_l
    =
    \left\lfloor h_l^* \right\rfloor,
    \qquad
    R
    =
    h-\sum_{l=1}^{L}\widehat{h}_l,
    \label{eq:rounding_floor}
\end{equation}
and assign the remaining $R$ heads to the layers with the largest fractional residues $h_l^*-\lfloor h_l^*\rfloor$. This preserves the total budget exactly and is computed once before federated training.

The allocation is optimal for the concave utility in~\eqref{eq:waterfill_opt}. Therefore, for the uniform allocation $h_l^{\mathrm{unif}}=h/L$,
\begin{equation}
    \sum_{l=1}^{L}
    a_l\log(1+h_l^*)
    \geq
    \sum_{l=1}^{L}
    a_l\log(1+h_l^{\mathrm{unif}}),
    \label{eq:waterfill_dominates}
\end{equation}
with strict inequality whenever the uniform allocation violates the KKT conditions and the layer capacities are not identical. This provides a principled non-uniform allocation rule without adding runtime or communication overhead.
\subsection{Server-Side Aggregation}
\label{subsec:agg}

After local training, each client uploads the compressed products $\{s_{c,l,i}\mathbf{H}_{c,l,i}\}$ for the heads it updated. Let $\mathcal{C}_{l,i}^{(t)}$ be the set of clients that updated head $i$ in layer $l$ at round $t$. The server aggregates the uploaded cores using data-size weights
\begin{equation}
    \mathbf{H}_{l,i}^{(t+1)}
    =
    \sum_{c\in\mathcal{C}_{l,i}^{(t)}}
    \omega_{c,l,i}^{(t)}
    s_{c,l,i}^{(t)}\mathbf{H}_{c,l,i}^{(t)},
    \qquad
    \omega_{c,l,i}^{(t)}
    =
    \frac{n_c}
    {\sum_{j\in\mathcal{C}_{l,i}^{(t)}}n_j}.
    \label{eq:server_aggregation}
\end{equation}
When all clients have equal local sample sizes,~\eqref{eq:server_aggregation} reduces to the equal-weight average used in the controlled simulations. The aggregation is head-wise, so clients with limited compute may update only a subset of heads without changing the protocol.

The exactness of aggregation follows from the shared-basis structure. For any fixed set of weights $\omega_c$ satisfying $\sum_c\omega_c=1$,
\begin{align}
    &\sum_{c\in\mathcal{C}^{(t)}}\omega_c
    \sum_{i=1}^{h_l}
    \mathbf{B}_{l,i}
    (s_{c,l,i}\mathbf{H}_{c,l,i})
    \mathbf{A}_{l,i}
    \nonumber\\
    &=
    \sum_{i=1}^{h_l}
    \mathbf{B}_{l,i}
    \left(
    \sum_{c\in\mathcal{C}^{(t)}}\omega_c
    s_{c,l,i}\mathbf{H}_{c,l,i}
    \right)
    \mathbf{A}_{l,i}.
    \label{eq:exact_agg_server}
\end{align}
Thus the server recovers the exact weighted mean of the client updates in the ambient parameter space. This exactness removes algebraic factorization bias; it does not by itself eliminate statistical heterogeneity, which is why TRISHUL also applies client-side spectral shrinkage. After aggregation, all server-side scalars are reset to $1$ to avoid round-to-round accumulation effects and participation-frequency bias.

\subsection{Overall Algorithm}

Algorithm~\ref{alg:trishul} summarizes TRISHUL. The algorithm explicitly separates one-time allocation and basis initialization from per-round client optimization and server aggregation. SVT is applied after each local gradient step on active heads; in practice, it can also be applied every few local steps to reduce overhead when $r$ is large.

\begin{algorithm}
\caption{TRISHUL: Three-Pronged Spectral Control for Federated Edge Fine-Tuning}
\label{alg:trishul}
\begin{algorithmic}[1]
\Require Clients $\mathcal{C}$, layers $L$, total head budget $h$, rank $r$, rounds $T$, local steps $S$, learning rate $\eta$, shrinkage $\lambda$, scalar bound $s_{\max}$.
\Statex \textit{// Server initialization}
\State Compute layer scores $a_l=\|\mathbf{W}_l^0\|_F^2+\epsilon$.
\State Solve~\eqref{eq:waterfill_opt} and obtain integer $\{h_l\}_{l=1}^{L}$ using~\eqref{eq:rounding_floor}.
\State Initialize orthonormal bases $\mathbf{B}_{l,i},\mathbf{A}_{l,i}$; set $\mathbf{H}_{l,i}\leftarrow\mathbf{0}$ and $s_{l,i}\leftarrow1$.
\State Freeze $\mathbf{W}_l^0$, $\mathbf{B}_{l,i}$, and $\mathbf{A}_{l,i}$.
\For{$t=1,\ldots,T$}
    \State Sample participating clients $\mathcal{C}^{(t)}\subseteq\mathcal{C}$ and broadcast $\{\mathbf{H}_{l,i},s_{l,i}\}$.
    \For{each client $c\in\mathcal{C}^{(t)}$ in parallel}
        \State Select active-head mask $m_{c,l,i}^{(t)}\in\{0,1\}$ according to local compute budget.
        \State Initialize $\mathbf{H}_{c,l,i}\leftarrow\mathbf{H}_{l,i}$ and $s_{c,l,i}\leftarrow s_{l,i}$.
        \For{$\tau=1,\ldots,S$}
            \For{each active head $(l,i)$ with $m_{c,l,i}^{(t)}=1$}
                \State $\mathbf{H}_{c,l,i}^{\mathrm{grad}}\leftarrow \mathbf{H}_{c,l,i}-\eta\nabla_{\mathbf{H}_{c,l,i}}\mathcal{L}_c$.
                \State Compute SVD $\mathbf{H}_{c,l,i}^{\mathrm{grad}}=\mathbf{U}\boldsymbol{\Sigma}\mathbf{V}^{\top}$.
                \State $\mathbf{H}_{c,l,i}\leftarrow \mathbf{U}(\boldsymbol{\Sigma}-\lambda\eta\mathbf{I})_+\mathbf{V}^{\top}$.
                \State $s_{c,l,i}\leftarrow \operatorname{clip}(s_{c,l,i}-\eta\nabla_{s_{c,l,i}}\mathcal{L}_c,0,s_{\max})$.
            \EndFor
        \EndFor
        \State Upload $\{m_{c,l,i}^{(t)}s_{c,l,i}\mathbf{H}_{c,l,i}\}$.
    \EndFor
    \For{each head $(l,i)$ with $\mathcal{C}_{l,i}^{(t)}\neq\emptyset$}
        \State Aggregate $\mathbf{H}_{l,i}^{(t+1)}$ using~\eqref{eq:server_aggregation}.
        \State Reset $s_{l,i}\leftarrow1$.
    \EndFor
\EndFor
\end{algorithmic}
\end{algorithm}

\subsection{Complexity, Privacy Scope, and Deployment Considerations}
\label{subsec:complexity_privacy}

TRISHUL has the same per-round communication order as the underlying multi-head PEFT protocol. For layer $l$, each active client uploads and downloads $h_l r^2$ trainable core parameters, so the total communicated trainable state is
\begin{equation}
    \mathcal{O}\!\left(\sum_{l=1}^{L}h_l r^2\right).
    \label{eq:comm_complexity}
\end{equation}
The additional local computation comes from SVD on $r\times r$ cores, yielding $\mathcal{O}(h_l r^3)$ per layer per proximal step. Since $r\ll d_l$, this overhead is small compared with the forward and backward passes through the frozen pretrained model.

TRISHUL follows the standard FL assumption that raw client data are not shared. However, uploaded core matrices may still leak distributional or task-specific information; nuclear-norm shrinkage should not be interpreted as a formal privacy mechanism. TRISHUL is compatible with secure aggregation, update clipping, client-level differential privacy, encrypted aggregation, and post-quantum-secure communication, but these mechanisms are orthogonal to the spectral-control contribution and are left for future work.

The method naturally supports edge heterogeneity through active-head masks: constrained clients may update fewer heads, while the server aggregates each head over the clients that actually updated it. Real deployments may additionally require quantized uploads, straggler handling, intermittent connectivity support, and energy-aware client selection.

\section{Experiments}
\label{sec:experiments}

We evaluate \textit{TRISHUL} on a diverse set of vision and language tasks under federated settings with both statistical and computational heterogeneity. Our goal is to test whether explicit spectral control improves robustness, convergence, and final performance in realistic federated PEFT regimes. We compare TRISHUL against strong and representative federated PEFT baselines, including FedIT~\cite{zhang2024}, FedEx-LoRA~\cite{singhal2024}, FFA-LoRA~\cite{sun2024}, Fed-SB~\cite{singhal2025}, RAVAN~\cite{raje2025ravan}, and SCAFFOLD+LoRA~\cite{karimireddy2020scaffold}, where the latter combines a classical variance-reduction strategy with LoRA. To ensure fair comparison, we follow the same experimental protocol as RAVAN wherever possible, including datasets, models, federated partitioning, and optimization settings, while varying only the TRISHUL-specific spectral regularization parameter $\lambda$ and retaining the proposed concave water-filling allocation and exact multi-head aggregation design.

\subsection{Experimental Setup}
\label{subsec:setup}

\subsubsection{Datasets and Models}

For image classification, we use ViT-B/16~\cite{dosovitskiy2021} with 85M parameters on CIFAR-100~\cite{krizhevsky2009} and SVHN~\cite{netzer2011}. CIFAR-100 contains 50K training and 10K test images across 100 classes, while SVHN contains 73K training and 26K test samples across 10 classes. These two benchmarks provide complementary visual regimes: CIFAR-100 is more semantically diverse and structurally difficult, whereas SVHN is less class-diverse but still useful for testing robustness under heterogeneity.

For language tasks, we fine-tune T5-Base~\cite{raffel2020} with 224M parameters on 20~Newsgroups~\cite{mitchell1997} and MRQA~\cite{fisch2019}. 20~Newsgroups contains approximately 11K training and 7.5K test documents spanning 20 topics, while MRQA contains 516K training and 58K test examples aggregated across six question-answering sources. These tasks test TRISHUL under both document classification and extractive question-answering settings. To assess scalability to foundation-scale models, we additionally evaluate on the GLUE benchmark~\cite{wang2018} using LLaMA3.2-1B~\cite{grattafiori2024}.

\subsubsection{Federated Partitioning}

We simulate federated learning with $|\mathcal{C}| \in \{20,50\}$ clients. Under IID partitioning, each client receives an equal-sized random subset of the training data. Under non-IID partitioning, client-specific class proportions are sampled from a Dirichlet distribution with concentration parameter $\alpha = 0.3$. Smaller values of $\alpha$ produce stronger heterogeneity, so we additionally consider $\alpha \in \{0.1,0.05\}$ in Section~\ref{subsubsec:extreme_heterogeneity}. For MRQA, the Dirichlet split is applied over the six constituent sub-datasets. In each round, the server uniformly samples three clients, and each selected client performs $S=50$ local iterations. The number of communication rounds is set to 50 for vision tasks, 100 for 20~Newsgroups, and 20 for MRQA and GLUE. Batch sizes and sequence lengths follow the settings used in RAVAN~\cite{raje2025ravan}.

\subsubsection{Parameter Budget and Hyperparameter}

We evaluate two trainable parameter budgets. The lower budget matches the trainable-parameter count of vanilla LoRA with rank $r=32$, and the higher budget matches vanilla LoRA with rank $r=64$. For RAVAN and TRISHUL, the trainable state consists of $h=4$ heads with core matrices $\mathbf{H}_i\in\mathbb{R}^{r\times r}$, with the total number of trainable parameters matched to the corresponding baseline budget. In Tables~\ref{tab:vision} and~\ref{tab:language}, the column \textit{Budget} denotes the parameter-budget-equivalent rank reported by each method; it should not be interpreted as the literal matrix rank for all baselines. This ensures that performance differences primarily reflect spectral control and allocation quality rather than increased capacity.

For TRISHUL, the spectral regularization coefficient is selected by grid search over $\lambda \in \{0.001, 0.01, 0.1, 1.0\}$, with $\lambda = 0.01$ providing the best and most stable performance across tasks. The full sensitivity study is reported in Section~\ref{subsubsec:lambda_sensitivity}. Learning rates are tuned per method and per budget. Unless otherwise stated, all results are averaged over three random seeds.

\subsubsection{Baseline implementation}
For a fair and controlled comparison, we implemented all baselines within the same experimental codebase as TRISHUL rather than relying on separately released implementations. Each baseline was reproduced from its corresponding paper and evaluated under the same federated protocol, data partitioning strategy, client sampling procedure, model backbone, parameter budget, number of communication rounds, local update steps, and evaluation pipeline used for TRISHUL. Method-specific hyperparameters were tuned following the ranges reported in the original papers where available, while shared training settings were kept identical across methods.
This ensures that performance differences primarily reflect the adaptation and aggregation mechanisms of each method rather than implementation-level or experimental-protocol differences.

\emph{Implementation details for reproducibility.}
All reported results are averaged over three independent random seeds. Learning rates, weight decay, local batch sizes, sequence lengths, LoRA target modules, and method-specific hyperparameter ranges are kept fixed across methods within each dataset-budget setting. For GLUE with LLaMA3.2-1B, the same prompt format, label verbalizers, maximum sequence length, and target adaptation modules are used across all PEFT baselines. Detailed seed values, hardware configuration, optimizer settings, and hyperparameter ranges will be released with the code to ensure exact reproducibility.

\begin{figure}
    \centering
    \includegraphics[width=\columnwidth]{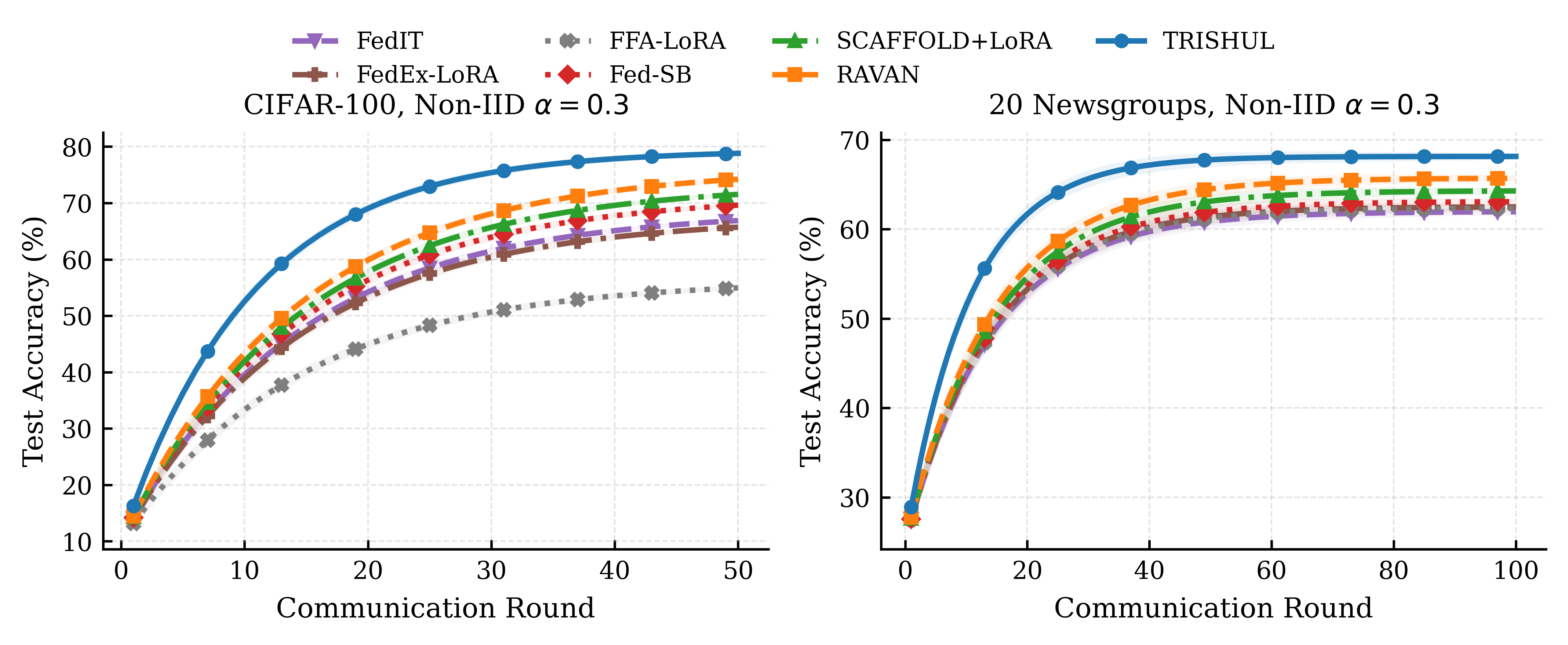}
    \caption{Federated training curves on CIFAR-100 (left) and 20~Newsgroups (right) under non-IID data ($\alpha=0.3$) with 20 clients and the lower parameter budget ($r=32$). TRISHUL converges faster and reaches a higher final accuracy than all baselines in both settings. % $79.41\%$ vs.\ RAVAN's $76.22\%$ on CIFAR-100, and $68.14\%$ vs.\ $65.73\%$ on 20~Newsgroups. The performance gap opens in the early communication rounds, demonstrating that nuclear-norm proximal shrinkage stabilizes aggregation from the outset of training.
    }
\label{fig:convergence}
\end{figure}

\subsection{Main Results: Vision and Language}
\label{subsec:main}

\subsubsection{Vision and Language Benchmarks}

Tables~\ref{tab:vision} and~\ref{tab:language} compare TRISHUL against all baselines on vision and language tasks. Across all datasets and parameter budgets, TRISHUL achieves the best performance, with gains that become more pronounced as client heterogeneity increases. This is consistent with the central claim of the paper: explicit spectral control is most valuable when local client updates are structurally inconsistent.

On CIFAR-100, TRISHUL improves over the strongest competing baseline by $3.2$--$3.5\%$ under non-IID partitions and by $0.7$--$0.8\%$ under IID partitions. On SVHN, the gains are smaller under IID but still consistent, while under non-IID they increase to as much as $2.8\%$. On 20~Newsgroups, TRISHUL delivers non-IID gains of roughly $2.4$--$2.5\%$, substantially larger than its IID gains, again indicating that the method is particularly effective when update spectra diverge across clients. On MRQA, the gains are smaller in absolute magnitude, but remain consistent across budgets and client counts, demonstrating that the method remains beneficial even on more structured language tasks.

A key observation is that SCAFFOLD+LoRA, despite explicitly correcting optimization drift, remains consistently below TRISHUL under heterogeneous settings. This suggests that gradient-level stabilization alone is insufficient for federated PEFT when the spectral geometry of updates is poorly aligned. By contrast, TRISHUL directly regularizes the update spectrum, suppressing unstable client-specific directions before aggregation.

Figure~\ref{fig:convergence} reports test accuracy versus communication rounds for CIFAR-100 and 20~Newsgroups under non-IID conditions with 20 clients and the lower parameter budget ($r=32$). In CIFAR-100, TRISHUL reaches $79.41\%$ by round 50, while the strongest baseline RAVAN plateaus at $76.22\%$; the gap opens as early as round 10, indicating that proximal shrinkage of the nuclear-norm reduces destructive interference from aggregation from the first communication rounds. On 20~Newsgroups, TRISHUL converges to $68.14\%$ versus RAVAN's $65.73\%$ by round 100, with a similarly early separation. In both settings, SCAFFOLD+LoRA, despite explicit gradient-variance correction, converges more slowly than TRISHUL, suggesting that gradient-level stabilization alone is insufficient when client updates are spectrally misaligned.

\subsubsection{Scaling to LLaMA3.2-1B on GLUE}

Table~\ref{tab:glue} shows that the benefits of TRISHUL persist at foundation-model scale. On the GLUE benchmark with LLaMA3.2-1B, TRISHUL achieves an average score of $84.66\%$, outperforming all baselines including SCAFFOLD+LoRA and Fed-SB. The largest individual gain occurs on RTE, where TRISHUL reaches $68.05\%$, the best result among all compared methods.

These results suggest that the proposed spectral control mechanism is not limited to moderate-scale models or a particular modality. Instead, it appears to generalize across model families, tasks, and scales. This is important because large-model federated fine-tuning is precisely where communication efficiency and stable adaptation become most critical.

Figure~\ref{fig:overhead} shows that TRISHUL's improvements come at negligible additional computational cost. Across ViT-B/16 (85M), T5-Base (224M), and LLaMA3.2-1B (1B), the per-round training time of TRISHUL remains close to that of RAVAN, with overheads of only $1.09\%$, $0.96\%$, and $0.83\%$, respectively. The absolute per-round training times are 18.6\,s, 31.5\,s, and 85.4\,s for TRISHUL, compared with 18.4\,s, 31.2\,s, and 84.7\,s for RAVAN. This small increase is seen because TRISHUL performs an additional proximal SVD step during local optimization. Since the SVD is applied only to the small $r \times r$ core matrices rather than to the full weight tensors, the additional computational cost is bounded by $\mathcal{O}(r^3)$ per head per proximal step and remains insignificant at billion-parameter scale.

\begin{figure}
    \centering
    \includegraphics[width=0.8\columnwidth]{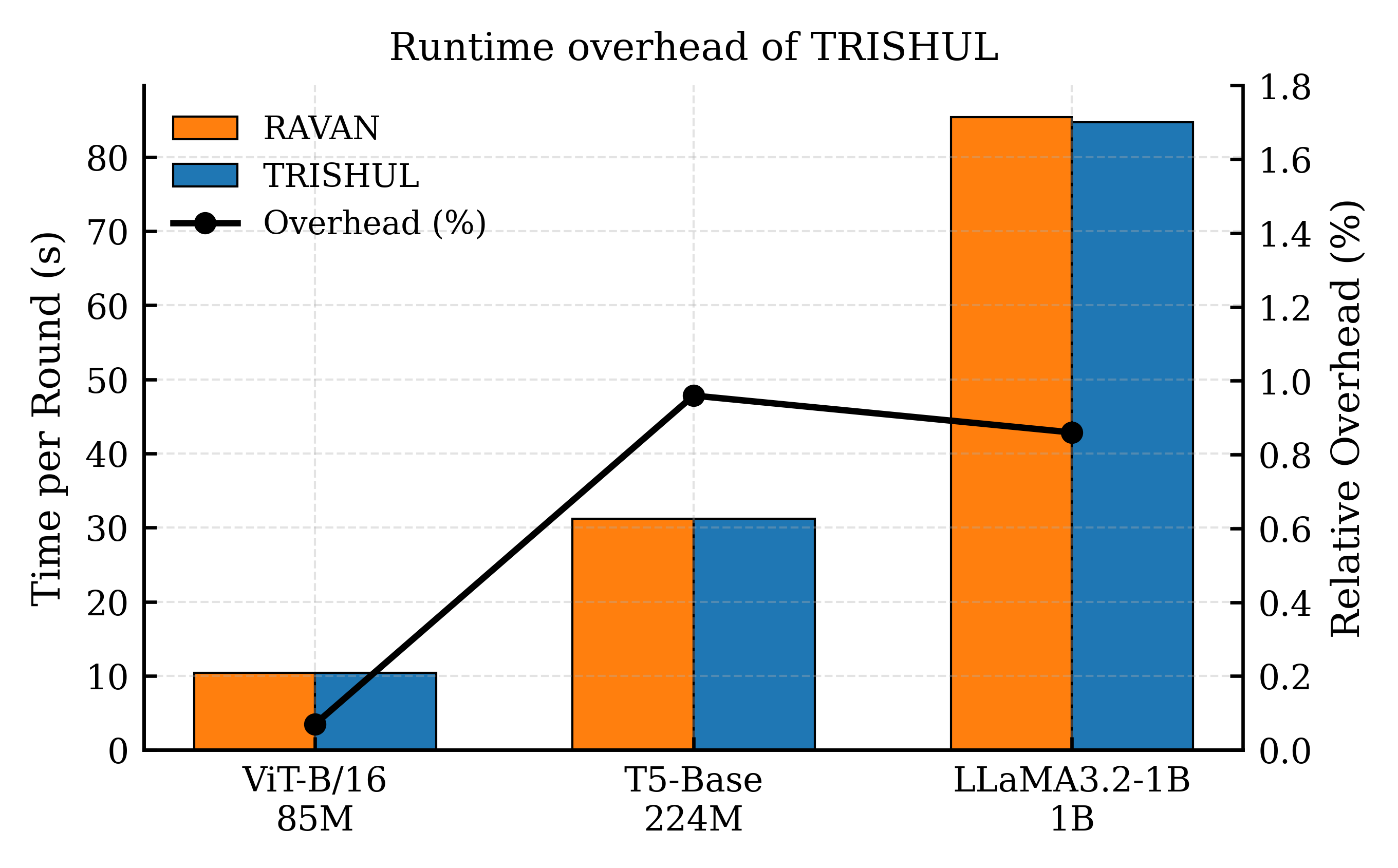}
    \caption{Training time per communication round for TRISHUL and RAVAN~\cite{raje2025ravan} across
    three model scales: ViT-B/16 with 85M parameters, T5-Base with 224M
    parameters, and LLaMA3.2-1B with 1B parameters. The secondary axis shows
    the relative overhead of TRISHUL, which remains approximately around
    $1\%$ or lower across all settings ($1.09\%$, $0.96\%$, and $0.83\%$,
    respectively). %The overhead is negligible because the proximal SVD is applied only to the small $r \times r$ core matrices, rather than to the full weight tensors.
    }
    \label{fig:overhead}
\end{figure}

\begin{table}
\caption{Performance comparison on CIFAR-100 and SVHN; Budget denotes the parameter-budget-equivalent rank, not necessarily the literal matrix rank for every baseline.}
\label{tab:vision}
\centering
\resizebox{\columnwidth}{!}{%
\begin{tabular}{l c cc cc cc cc}
\toprule
\multirow{3}{*}{Method} & \multirow{3}{*}{Budget} 
& \multicolumn{4}{c}{CIFAR-100 (Acc.\%)} 
& \multicolumn{4}{c}{SVHN (Acc.\%)} \\
\cmidrule(lr){3-6} \cmidrule(lr){7-10}
& & \multicolumn{2}{c}{20 clients} & \multicolumn{2}{c}{50 clients} 
  & \multicolumn{2}{c}{20 clients} & \multicolumn{2}{c}{50 clients} \\
\cmidrule(lr){3-4} \cmidrule(lr){5-6} \cmidrule(lr){7-8} \cmidrule(lr){9-10}
& & IID & Non-IID & IID & Non-IID & IID & Non-IID & IID & Non-IID \\
\midrule
Full-FT        & N/A & 89.89 & 86.86 & 89.78 & 85.17 & 95.06 & 90.29 & 94.90 & 89.49 \\
FedIT          & 32  & 83.49 & 68.66 & 81.75 & 68.15 & 88.66 & 84.00 & 91.67 & 77.53 \\
FedEx-LoRA     & 32  & 80.56 & 67.45 & 77.82 & 66.58 & 91.94 & 84.30 & 91.51 & 81.63 \\
FFA-LoRA       & 64  & 78.82 & 56.34 & 78.17 & 59.98 & 91.53 & 86.03 & 91.82 & 83.30 \\
Fed-SB         & 221 & 79.27 & 71.48 & 79.06 & 69.51 & 90.94 & 82.25 & 92.74 & 85.30 \\
SCAFFOLD+LoRA  & 32  & 82.17 & 73.45 & 81.93 & 72.18 & 92.31 & 87.64 & 92.85 & 86.47 \\
RAVAN          & 110 & 84.42 & 76.22 & 84.02 & 73.80 & 94.13 & 90.02 & 93.75 & 89.17 \\
\textit{TRISHUL} & 110 & \textbf{85.13} & \textbf{79.41} & \textbf{84.68} & \textbf{76.94} 
                        & \textbf{94.72} & \textbf{91.38} & \textbf{94.21} & \textbf{90.33} \\
\midrule
FedIT          & 64  & 83.82 & 71.01 & 84.04 & 73.23 & 91.39 & 84.68 & 92.06 & 79.31 \\
FedEx-LoRA     & 64  & 79.38 & 50.47 & 79.42 & 57.86 & 91.16 & 74.04 & 92.01 & 74.84 \\
FFA-LoRA       & 128 & 81.39 & 70.31 & 82.13 & 66.81 & 91.95 & 88.06 & 92.07 & 84.24 \\
Fed-SB         & 313 & 83.03 & 73.12 & 83.90 & 71.13 & 92.29 & 86.89 & 92.78 & 82.46 \\
SCAFFOLD+LoRA  & 64  & 83.54 & 74.82 & 83.21 & 73.96 & 92.88 & 88.41 & 93.14 & 87.23 \\
RAVAN          & 156 & 85.04 & 77.20 & 85.55 & 77.81 & 93.92 & 89.41 & 94.28 & 84.34 \\
\textit{TRISHUL} & 156 & \textbf{85.82} & \textbf{80.67} & \textbf{86.31} & \textbf{81.24} 
                        & \textbf{94.58} & \textbf{91.02} & \textbf{94.83} & \textbf{87.11} \\
\bottomrule
\end{tabular}}
\end{table}

\begin{table}
\caption{Performance comparison on 20~Newsgroups and MRQA. ``Budget'' denotes the parameter-budget-equivalent rank, not necessarily the literal matrix rank for every baseline.}
\label{tab:language}
\centering
\resizebox{\columnwidth}{!}{%
\begin{tabular}{l c cc cc cc cc}
\toprule
\multirow{3}{*}{Method} & \multirow{3}{*}{Budget} 
& \multicolumn{4}{c}{20 Newsgroups (Acc.\%)} 
& \multicolumn{4}{c}{MRQA (F1\%)} \\
\cmidrule(lr){3-6} \cmidrule(lr){7-10}
& & \multicolumn{2}{c}{20 clients} & \multicolumn{2}{c}{50 clients} 
  & \multicolumn{2}{c}{20 clients} & \multicolumn{2}{c}{50 clients} \\
\cmidrule(lr){3-4} \cmidrule(lr){5-6} \cmidrule(lr){7-8} \cmidrule(lr){9-10}
& & IID & Non-IID & IID & Non-IID & IID & Non-IID & IID & Non-IID \\
\midrule
Full-FT        & N/A & 71.34 & 69.29 & 71.71 & 70.13 & 62.19 & 62.25 & 62.41 & 62.51 \\
FedIT          & 32  & 69.07 & 61.98 & 67.99 & 60.67 & 61.00 & 60.57 & 61.24 & 60.52 \\
FedEx-LoRA     & 32  & 69.04 & 62.52 & 68.19 & 63.33 & 60.99 & 60.68 & 61.40 & 60.56 \\
FFA-LoRA       & 64  & 68.11 & 62.36 & 68.00 & 64.86 & 60.31 & 60.40 & 61.21 & 60.14 \\
Fed-SB         & 221 & 67.15 & 63.10 & 66.69 & 63.98 & 59.93 & 59.73 & 59.96 & 60.01 \\
SCAFFOLD+LoRA  & 32  & 68.74 & 64.31 & 67.85 & 63.92 & 61.04 & 60.71 & 61.18 & 60.83 \\
RAVAN          & 110 & 68.96 & 65.73 & 68.18 & 65.67 & 61.18 & 60.45 & 61.33 & 61.53 \\
\textit{TRISHUL} & 110 & \textbf{69.82} & \textbf{68.14} & \textbf{68.93} & \textbf{68.02} 
                        & \textbf{61.72} & \textbf{61.23} & \textbf{61.86} & \textbf{62.27} \\
\midrule
FedIT          & 64  & 69.36 & 64.41 & 68.12 & 62.67 & 61.25 & 60.75 & 61.39 & 60.26 \\
FedEx-LoRA     & 64  & 68.59 & 65.11 & 67.75 & 64.31 & 61.23 & 60.36 & 61.43 & 60.06 \\
FFA-LoRA       & 128 & 69.33 & 66.22 & 68.42 & 64.86 & 61.50 & 60.50 & 61.66 & 60.12 \\
Fed-SB         & 313 & 68.07 & 64.18 & 67.58 & 65.59 & 60.22 & 60.11 & 60.28 & 60.60 \\
SCAFFOLD+LoRA  & 64  & 69.11 & 65.84 & 68.43 & 65.17 & 61.38 & 60.94 & 61.52 & 61.04 \\
RAVAN          & 156 & 69.29 & 66.45 & 68.89 & 66.85 & 61.82 & 61.33 & 61.73 & 61.26 \\
\textit{TRISHUL} & 156 & \textbf{70.18} & \textbf{68.94} & \textbf{69.77} & \textbf{69.31} 
                        & \textbf{62.35} & \textbf{62.01} & \textbf{62.28} & \textbf{62.09} \\
\bottomrule
\end{tabular}}
\end{table}

\begin{table}
\caption{GLUE benchmark results with LLaMA3.2-1B (20 clients, lower budget).}
\label{tab:glue}
\centering
\resizebox{\columnwidth}{!}{%
\begin{tabular}{l ccccccc}
\toprule
Method & MNLI-MM & MNLI-M & QNLI & QQP & SST-2 & RTE & Avg. \\
\midrule
FedIT          & 84.24 & 84.62 & 82.74 & 85.96 & 94.61 & 65.70 & 82.97 \\
FedEx-LoRA     & 84.15 & 84.70 & 82.74 & 86.07 & 94.61 & 65.34 & 82.94 \\
FFA-LoRA       & 85.05 & 85.78 & 82.07 & 84.40 & 94.38 & 62.46 & 82.36 \\
Fed-SB         & 84.88 & 85.23 & 82.84 & 84.23 & 94.95 & 67.15 & 83.21 \\
SCAFFOLD+LoRA  & 84.63 & 85.01 & 83.17 & 85.44 & 94.82 & 66.58 & 83.28 \\
RAVAN          & 85.24 & 85.65 & 84.00 & 86.11 & 95.18 & 67.15 & 83.90 \\
\textit{TRISHUL} & \textbf{85.74} & \textbf{86.15} & \textbf{84.75} & \textbf{86.71} 
                 & \textbf{95.58} & \textbf{68.05} & \textbf{84.66} \\
\bottomrule
\end{tabular}}
\end{table}

\subsection{Ablation Studies and Analysis}
\label{subsec:ablation}

We now examine the individual contributions of the three prongs of TRISHUL. Specifically, we isolate the effects of nuclear-norm spectral regularization, concave water-filling head allocation, trainable head scaling, basis initialization, and robustness under computational heterogeneity. Unless otherwise specified, all ablations are conducted with the lower parameter budget ($r=32$), 20 clients, and results averaged over three random seeds.

For spectral diagnostics, we use the following definitions. Given singular values $\{\sigma_j\}_{j=1}^{r}$ of a client core matrix, let $p_j=\sigma_j/\sum_k\sigma_k$. Spectral entropy is
\begin{equation}
    \mathcal{H}(\mathbf{H})
    =
    -\sum_j p_j\log p_j,
    \label{eq:spectral_entropy}
\end{equation}
and the effective rank is $r_{\mathrm{eff}}(\mathbf{H})=\exp(\mathcal{H}(\mathbf{H}))$. For two top-$k$ left singular subspaces $\mathbf{U}$ and $\mathbf{V}$, principal-angle similarity is
\begin{equation}
    S_{\mathrm{angle}}(\mathbf{U},\mathbf{V})
    =
    \frac{1}{k}\sum_{j=1}^{k}\cos^2\theta_j(\mathbf{U},\mathbf{V}),
    \label{eq:principal_similarity}
\end{equation}
where $\theta_j$ are the principal angles. Dominant-direction similarity is $S_{\mathrm{dom}}=|\mathbf{u}_1^\top\mathbf{v}_1|$. We also report normalized inter-client aggregation variance by dividing the empirical Frobenius variance of uploaded core products by the corresponding RAVAN value at the first communication round.

\subsubsection{Effect of Nuclear-Norm Regularization}
\label{subsubsec:nuclear_norm}

The nuclear-norm penalty is designed to suppress unstable high-rank local update directions before server aggregation. Table~\ref{tab:ablation_nuclear} compares TRISHUL with and without this penalty. Across all tasks, the regularized version performs better, with the most substantial gains under non-IID data. On CIFAR-100, the improvement under non-IID reaches $3.07\%$, which is consistent with the theoretical variance bound in~\eqref{eq:variance_bound}. These results support the hypothesis that local spectral shrinkage reduces destructive aggregation interference and improves transferability of client updates.

\begin{table}
\caption{Ablation: effect of nuclear-norm penalty (lower budget, 20 clients).}
\label{tab:ablation_nuclear}
\centering
\resizebox{\columnwidth}{!}{%
\begin{tabular}{l cc cc cc}
\toprule
\multirow{2}{*}{Setting} 
& \multicolumn{2}{c}{CIFAR-100} 
& \multicolumn{2}{c}{SVHN} 
& \multicolumn{2}{c}{20 Newsgroups} \\
\cmidrule(lr){2-3} \cmidrule(lr){4-5} \cmidrule(lr){6-7}
& IID & Non-IID & IID & Non-IID & IID & Non-IID \\
\midrule
w/o penalty (\(\lambda=0\))   & 84.12 & 76.34 & 94.02 & 90.21 & 68.87 & 65.96 \\
w/ penalty (\(\lambda=0.01\)) & \textbf{85.13} & \textbf{79.41} & \textbf{94.72} & \textbf{91.38} & \textbf{69.82} & \textbf{68.14} \\
\bottomrule
\end{tabular}}
\end{table}

\paragraph{Effect of singular-value thresholding}
To better understand why nuclear-norm shrinkage improves federated aggregation, we inspect the singular-value spectra before and after the proximal SVT step and compare the spectra of RAVAN and TRISHUL in Figure~\ref{fig:svt_spectra}. The left panel shows that, under non-IID data, the pre-SVT spectrum contains a heavy tail of weak singular values corresponding to client-specific adaptation directions that do not transfer across clients; averaging these directions inflates aggregation noise. After the SVT step with $\lambda=0.01$, singular values below the threshold $\tau=\lambda\eta$ are zeroed out entirely, producing a sharp spectrum concentrated on the dominant modes. The right panel compares RAVAN and TRISHUL directly: RAVAN retains substantially more tail energy, whereas TRISHUL's spectrum drops off steeply after the top few modes, indicating stronger spectral control. This compression is what enables TRISHUL to achieve lower aggregation variance and higher cross-client subspace alignment in Figure~\ref{fig:spectral_diagnostics} and Figure~\ref{fig:spectral_alignment}, respectively.

\begin{figure}[!t]
\centering
\includegraphics[width=\linewidth]{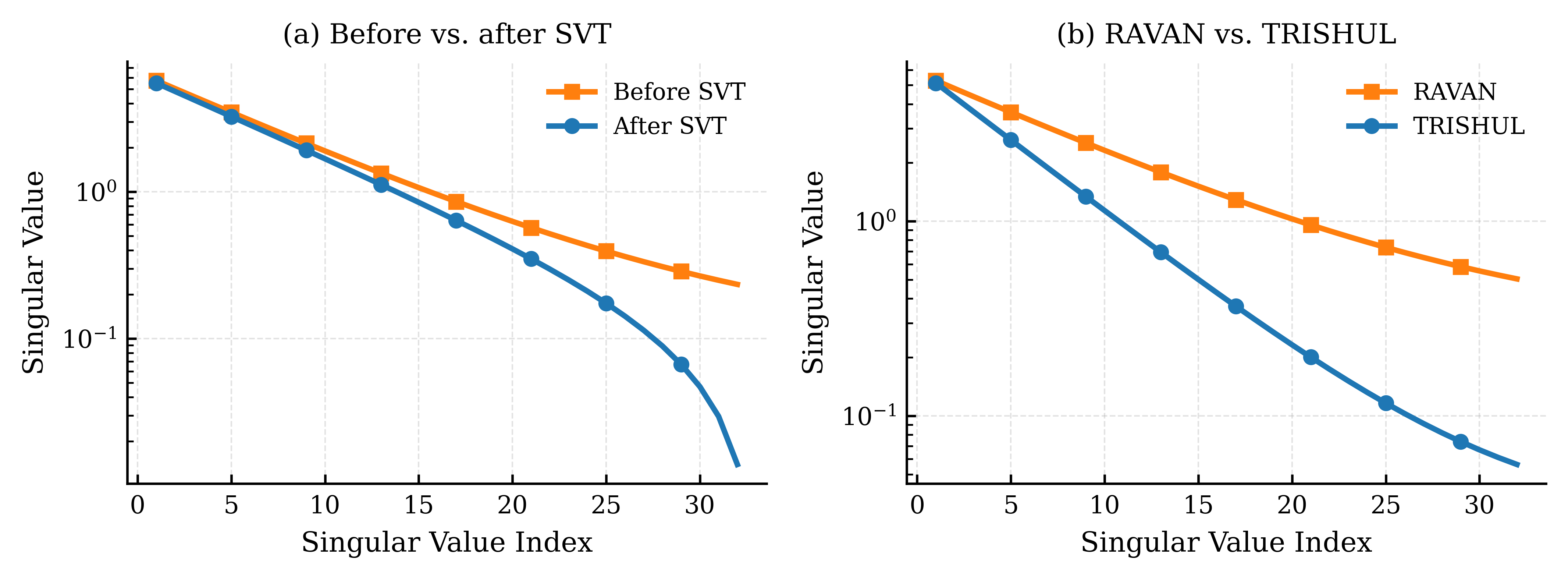}

\vspace{-0.4em}
\makebox[0.48\linewidth]
\hfill
\makebox[0.48\linewidth]

\caption{Singular-value spectra of core matrices ($\lambda=0.01$, $r=32$, CIFAR-100 non-IID). 
The left panel compares the spectrum before and after the proximal SVT step, showing that weak tail components corresponding to client-specific high-rank noise are zeroed out while dominant singular values are preserved. 
The right panel compares RAVAN and TRISHUL, showing that RAVAN retains a heavier tail beyond the top few modes, whereas TRISHUL produces a sharper spectrum with reduced tail energy.}
\label{fig:svt_spectra}
\end{figure}

\paragraph{Direct spectral inconsistency analysis}
To directly verify that TRISHUL reduces spectral inconsistency across heterogeneous clients, we track five complementary diagnostics computed from the client-side core updates before aggregation on CIFAR-100 under non-IID partitioning ($\alpha=0.3$, 20 clients). Figure~\ref{fig:spectral_alignment} shows that TRISHUL achieves consistently higher principal-angle similarity between client update subspaces across all 50 rounds, indicating that spectral shrinkage aligns the adaptation subspaces learned by different clients. Figure~\ref{fig:spectral_alignment} provides complementary evidence at the level of individual directions: the cosine similarity between dominant singular vectors is higher for TRISHUL throughout training, meaning the surviving spectral modes are more globally shared and less client-specific. Figure~\ref{fig:spectral_diagnostics} shows that TRISHUL's updates have lower spectral entropy, reflecting energy concentrated in a few stable modes rather than spread across many noisy directions. Figure~\ref{fig:spectral_diagnostics} shows that the effective rank of uploaded core matrices is consistently lower under TRISHUL than RAVAN, indicating that nuclear-norm shrinkage suppresses weak tail singular directions before server aggregation. Finally, Figure~\ref{fig:spectral_diagnostics} shows that the normalized inter-client aggregation variance is substantially lower under TRISHUL, with the gap widening over training rounds. The version of TRISHUL without shrinkage ($\lambda=0$) falls between RAVAN and TRISHUL on all five metrics, isolating the contribution of the proximal SVT step.

\begin{figure}[!t]
\centering
\includegraphics[width=\linewidth]{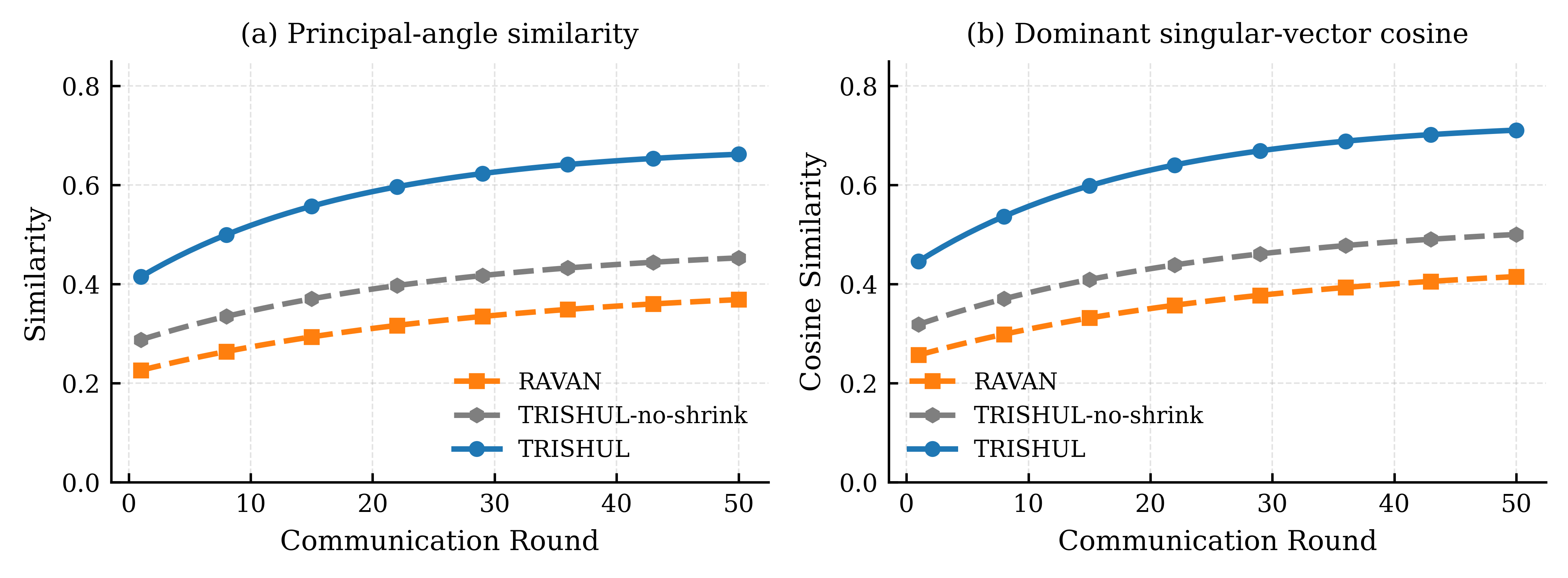}

\caption{Spectral alignment of client update subspaces across communication rounds (CIFAR-100, non-IID, $\alpha=0.3$, 20 clients). The left panel shows principal-angle similarity between client update subspaces:
TRISHUL achieves consistently higher subspace alignment than both RAVAN and TRISHUL without shrinkage throughout training. The right panel shows dominant singular-vector cosine similarity across clients: TRISHUL yields more aligned dominant spectral directions at every communication round.}
\label{fig:spectral_alignment}
\end{figure}

\begin{figure}[!t]
\centering
\includegraphics[width=\linewidth]{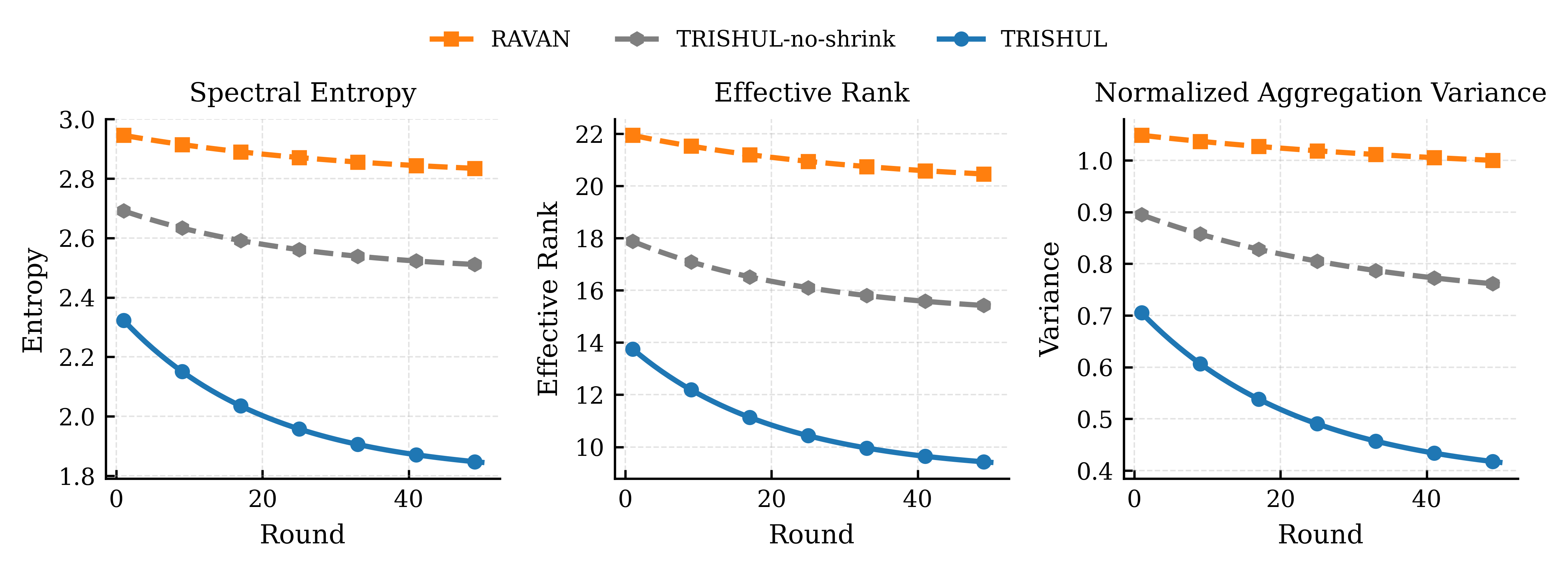}

\caption{Spectral diagnostics of client core updates across communication
rounds (CIFAR-100, non-IID, $\alpha=0.3$, 20 clients). 
The left panel shows spectral entropy: lower entropy under TRISHUL indicates
that update energy becomes increasingly concentrated in a small number of
stable dominant modes as training proceeds. 
The middle panel shows effective rank: TRISHUL maintains a substantially
lower effective rank than RAVAN and TRISHUL without shrinkage, indicating that nuclear-norm proximal thresholding removes weak tail singular directions before server aggregation. 
The right panel shows normalized inter-client aggregation variance: TRISHUL
achieves the lowest aggregation variance of the three compared settings,
with the gap over RAVAN widening as training progresses. 
Together, these diagnostics show that TRISHUL produces spectrally compact,
low-rank updates with reduced inter-client variance, consistent with the
rank-controlled variance mechanism in Eq.~\eqref{eq:variance_bound}.}
\label{fig:spectral_diagnostics}
\end{figure}

\subsubsection{Sensitivity to the Regularization Parameter $\lambda$}
\label{subsubsec:lambda_sensitivity}

Table~\ref{tab:lambda_sensitivity} and Figure~\ref{fig:lambda} reveal a
clear optimum at $\lambda=0.01$ across all three datasets. At $\lambda=0$ (no penalty), non-IID accuracy on CIFAR-100 is $76.34\%$; it rises to $79.41\%$ at $\lambda=0.01$ before falling back to $74.27\%$ at $\lambda=1.0$. The pattern is consistent on SVHN ($90.21\% \to 91.38\% \to 88.63\%$) and 20~Newsgroups ($65.96\% \to 68.14\% \to 63.71\%$). The method remains relatively stable across $\lambda\in[0.001,0.1]$, suggesting modest tuning burden in practice.

\begin{figure}
    \centering
    \includegraphics[width=\columnwidth]{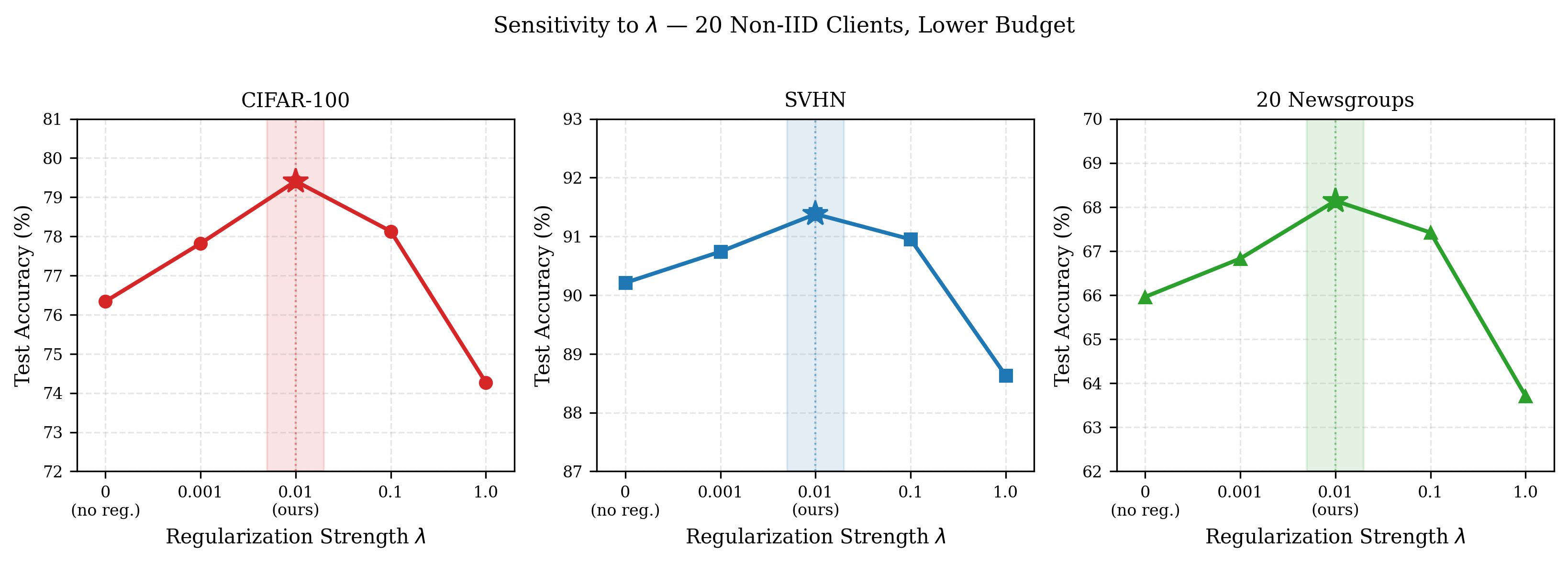}
    \caption{Test accuracy versus regularization strength $\lambda \in \{0,0.001,0.01,0.1,1.0\}$ on CIFAR-100, SVHN, and 20~Newsgroups under non-IID data with 20 clients (lower budget, $r=32$). A clear optimum occurs at $\lambda=0.01$ on all three datasets. Performance is relatively stable across $\lambda\in[0.001,0.1]$, indicating modest sensitivity to the exact choice of regularization strength.}
\label{fig:lambda}
\end{figure}

\begin{table}
\caption{Sensitivity to \(\lambda\) (lower budget, 20 non-IID clients).}
\label{tab:lambda_sensitivity}
\centering
\resizebox{\columnwidth}{!}{%
\begin{tabular}{l ccc}
\toprule
\(\lambda\) & CIFAR-100 (Acc.\%) & SVHN (Acc.\%) & 20 Newsgroups (Acc.\%) \\
\midrule
0 (no penalty)  & 76.34 & 90.21 & 65.96 \\
0.001           & 77.82 & 90.74 & 66.83 \\
0.01 (ours)     & \textbf{79.41} & \textbf{91.38} & \textbf{68.14} \\
0.1             & 78.13 & 90.95 & 67.42 \\
1.0             & 74.27 & 88.63 & 63.71 \\
\bottomrule
\end{tabular}}
\end{table}

\paragraph{Bias--variance tradeoff of spectral shrinkage}
Figure~\ref{fig:lambda_diagnostics} plots CIFAR-100 accuracy against $\lambda$ directly. Performance rises as $\lambda$ increases from $0$ to $0.01$, because shrinkage progressively removes client-specific noise directions that inflate aggregation variance. Beyond $\lambda=0.01$, accuracy declines sharply: at $\lambda=1.0$ the threshold $\tau=\lambda\eta$ is large enough to zero out task-relevant singular modes, introducing approximation bias that outweighs the variance reduction. Figure~\ref{fig:lambda_diagnostics} shows the corresponding spectral entropy: entropy decreases monotonically with $\lambda$, showing progressive compression of the update spectrum. At $\lambda=1.0$ the entropy is minimal but accuracy is worst, illustrating that over-compression removes useful adaptation capacity. Together these figures provide an empirical validation of the bias--variance tradeoff.

\begin{figure}[!t]
\centering
\includegraphics[width=\linewidth]{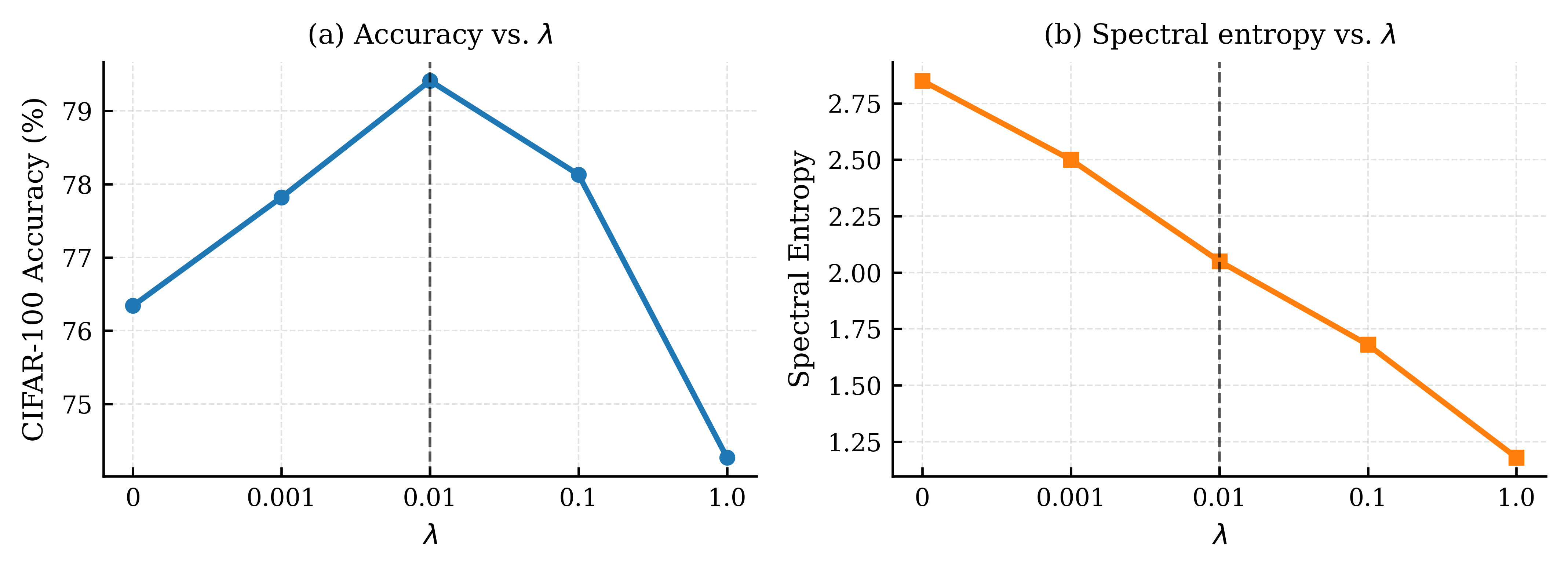}

\caption{Bias--variance tradeoff of spectral shrinkage on CIFAR-100
(non-IID, 20 clients, $r=32$). 
The left panel shows accuracy versus $\lambda$: moderate shrinkage
($\lambda=0.01$) improves accuracy by $3.07\%$ over no penalty, while
overly aggressive shrinkage ($\lambda=1.0$) reduces accuracy below the
unregularized baseline by zeroing out task-relevant singular modes and
introducing approximation bias. 
The right panel shows spectral entropy versus $\lambda$: entropy decreases
monotonically as $\lambda$ increases, reflecting progressive compression
of the update spectrum. The regime $\lambda\in[0.001,0.1]$ achieves the
best accuracy--entropy balance; beyond $\lambda=0.1$, entropy is suppressed
too aggressively, removing task-relevant directions and degrading downstream
performance.}
\label{fig:lambda_diagnostics}
\end{figure}

\subsubsection{Extreme Heterogeneity}
\label{subsubsec:extreme_heterogeneity}

To stress-test robustness, we evaluate all methods on CIFAR-100 with stronger non-IID partitions generated by $\alpha \in \{0.1,0.05\}$. Table~\ref{tab:extreme_heterogeneity} and Figure~\ref{fig:heterogeneity} show that TRISHUL's advantage widens monotonically as non-IID severity increases. At $\alpha=0.3$, TRISHUL leads the strongest baseline (RAVAN) by $3.2\%$ ($79.41\%$ vs.\ $76.22\%$). At $\alpha=0.1$ the gap grows to $4.3\%$ ($73.86\%$ vs.\ $69.54\%$), and at $\alpha=0.05$ it reaches $4.7\%$ ($68.53\%$ vs.\ $63.81\%$). Notably, FFA-LoRA collapses to $41.93\%$ at $\alpha=0.05$, and even SCAFFOLD+LoRA, the strongest gradient-level method—reaches only $62.14\%$, a gap of $6.4\%$ behind TRISHUL. This monotonically increasing advantage supports the central thesis: spectral inconsistency of client updates is a failure mode that gradient-only corrections may not resolve, and explicit spectral control via nuclear-norm proximal shrinkage becomes increasingly essential as client distributions diverge.

\begin{figure}
    \centering
    \includegraphics[width=\columnwidth]{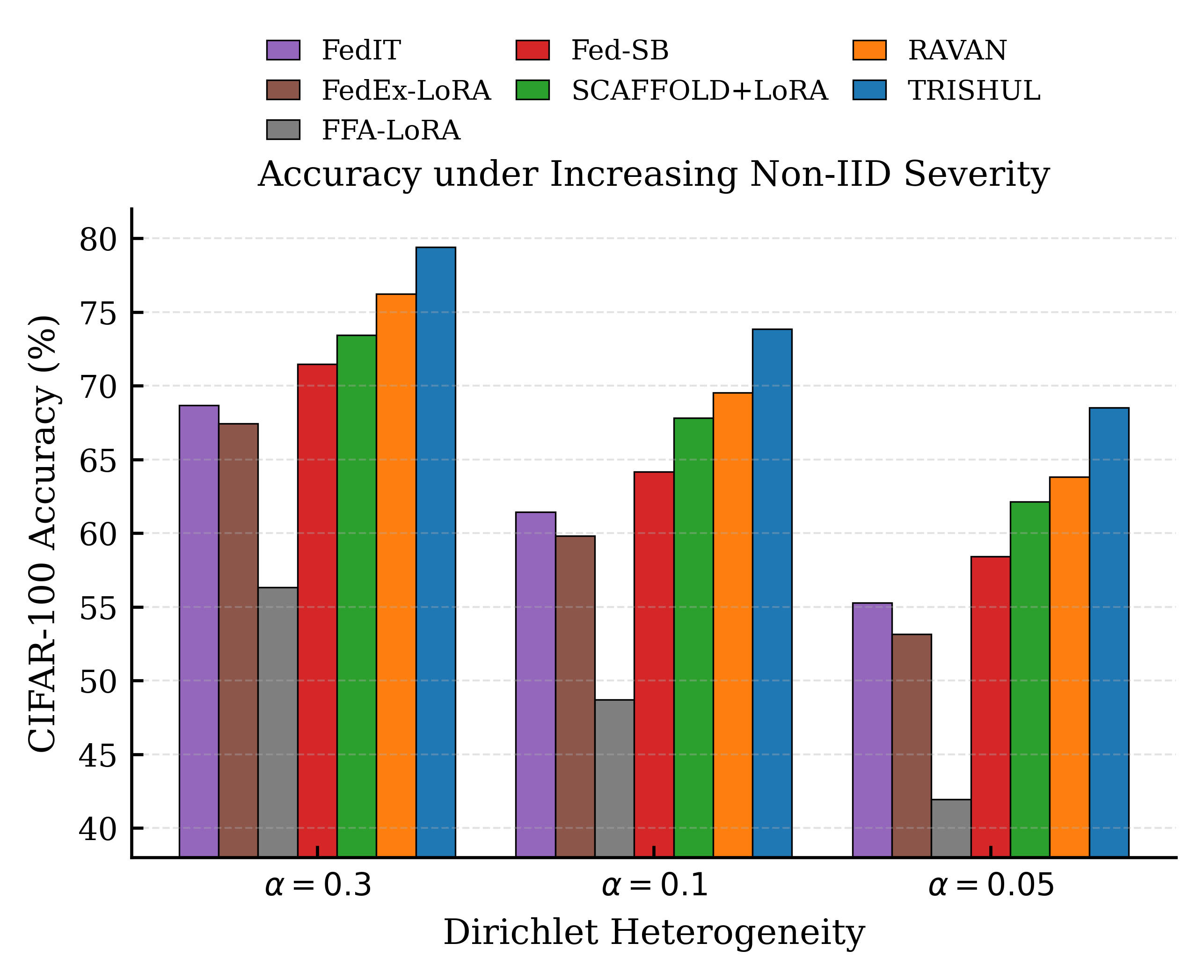}
    \caption{CIFAR-100 test accuracy under increasing non-IID severity ($\alpha \in \{0.3, 0.1, 0.05\}$) with 20 clients and the lower parameter budget ($r=32$). TRISHUL's advantage over the strongest baseline (RAVAN) grows from $3.2\%$ at $\alpha=0.3$ to $4.7\%$ at $\alpha=0.05$. SCAFFOLD+LoRA, despite explicit gradient-variance correction, trails TRISHUL by $6.4\%$ at $\alpha=0.05$, suggesting that gradient-level stabilization alone is insufficient under severe spectral misalignment of client updates.}
\label{fig:heterogeneity}
\end{figure}

\begin{table}
\caption{Accuracy under extreme heterogeneity (CIFAR-100, 20 clients, lower budget, non-IID).}
\label{tab:extreme_heterogeneity}
\centering
\resizebox{\columnwidth}{!}{%
\begin{tabular}{l ccc}
\toprule
Method & \(\alpha=0.3\) & \(\alpha=0.1\) & \(\alpha=0.05\) \\
\midrule
FedIT          & 68.66 & 61.43 & 55.27 \\
FedEx-LoRA     & 67.45 & 59.82 & 53.14 \\
FFA-LoRA       & 56.34 & 48.71 & 41.93 \\
Fed-SB         & 71.48 & 64.17 & 58.42 \\
SCAFFOLD+LoRA  & 73.45 & 67.82 & 62.14 \\
RAVAN          & 76.22 & 69.54 & 63.81 \\
\textbf{TRISHUL} & \textbf{79.41} & \textbf{73.86} & \textbf{68.53} \\
\bottomrule
\end{tabular}}
\end{table}

\paragraph{Spectral behavior under stronger heterogeneity}
Figure~\ref{fig:acc_vs_inconsistency} plots CIFAR-100 accuracy against a spectral inconsistency index for RAVAN and TRISHUL across $\alpha\in \{0.3,0.1,0.05\}$. Both methods move toward the upper-left as $\alpha$ increases (higher inconsistency, lower accuracy), but TRISHUL consistently occupies the upper-left relative to RAVAN at each heterogeneity level, higher accuracy at strictly lower spectral inconsistency, supporting the interpretation that improved spectral alignment contributes to TRISHUL's accuracy gains. Figure~\ref{fig:alpha_diagnostics} shows that principal-angle similarity falls for all methods as $\alpha$ decreases, but TRISHUL maintains substantially higher alignment at every level: at $\alpha=0.05$ TRISHUL's alignment is approximately $0.35$ versus $0.15$ for RAVAN and near zero for TRISHUL without shrinkage. Figure~\ref{fig:alpha_diagnostics} shows that TRISHUL maintains lower spectral entropy than the other two settings at all three heterogeneity levels, with the entropy gap largest at $\alpha=0.05$, precisely where accuracy gains are greatest.

\begin{figure}[!t]
\centering
\includegraphics[width=\linewidth]{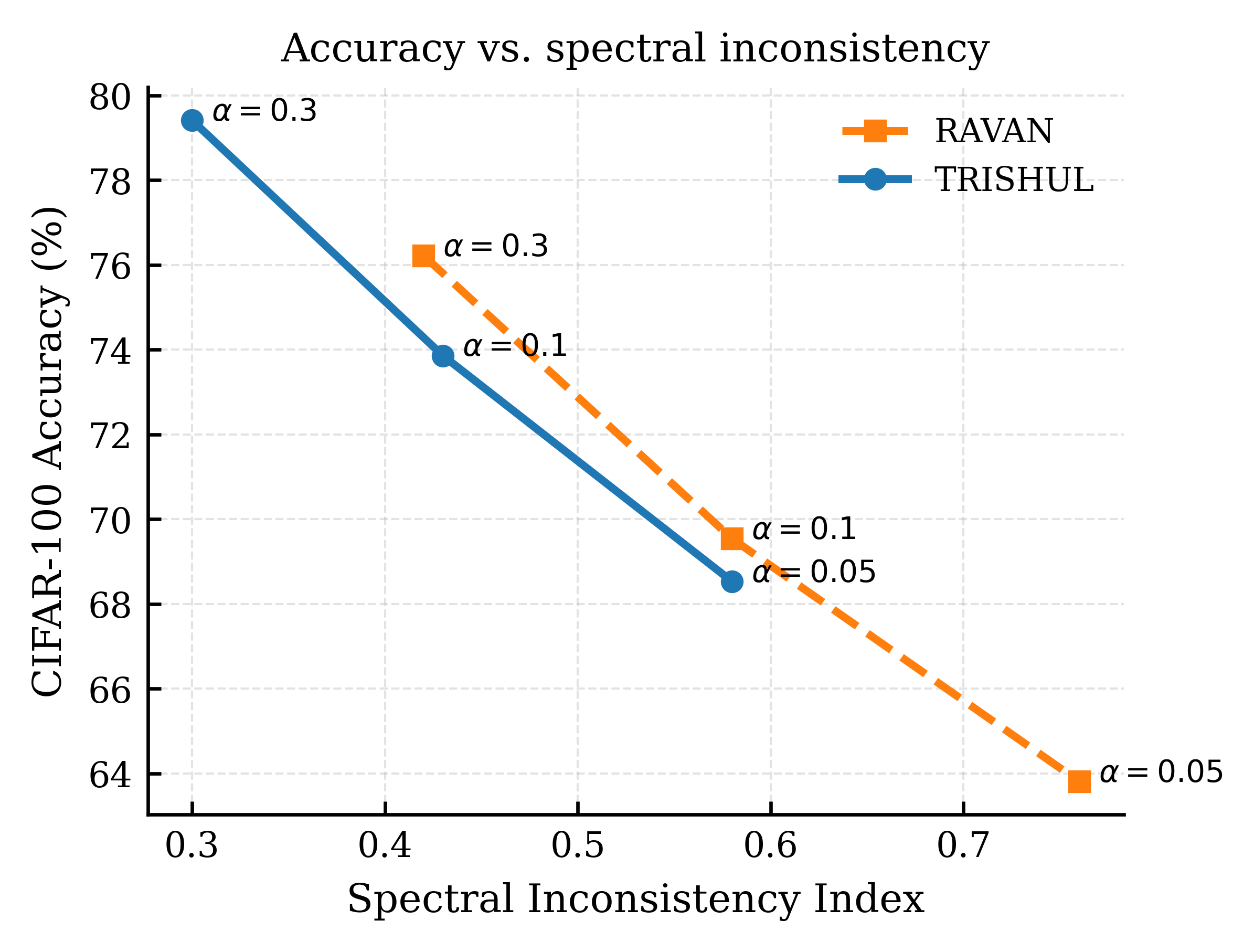}
\caption{Test accuracy versus spectral inconsistency index for
RAVAN and TRISHUL across $\alpha\in\{0.3,0.1,0.05\}$ on CIFAR-100 (20 clients, lower budget). At every heterogeneity level, TRISHUL achieves higher accuracy at lower spectral inconsistency, providing direct empirical evidence that improved spectral alignment across clients is a proximate cause of TRISHUL's performance gains.}
\label{fig:acc_vs_inconsistency}
\end{figure}

\begin{figure}[!t]
\centering
\includegraphics[width=\linewidth]{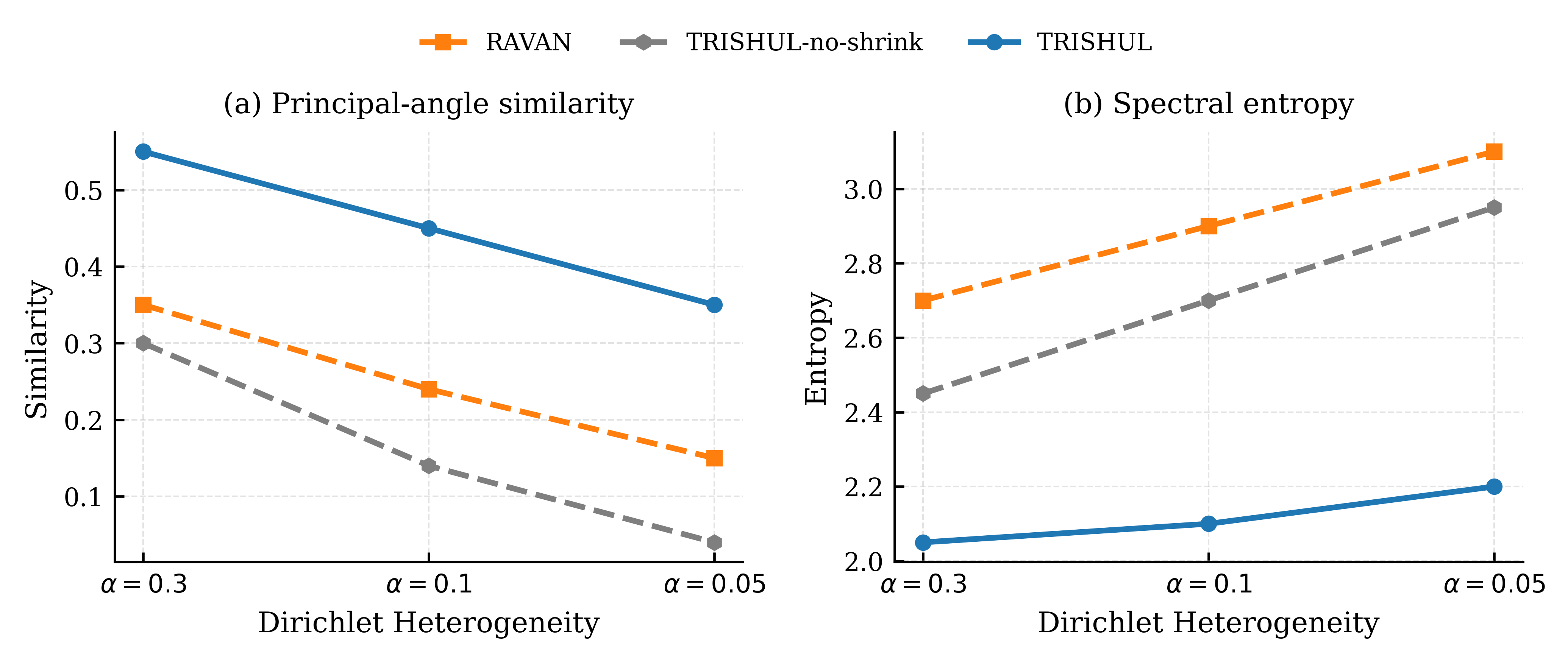}

\caption{Effect of increasing non-IID severity on spectral alignment and
entropy (CIFAR-100, 20 clients). The left panel shows principal-angle
similarity between client update subspaces under different Dirichlet
heterogeneity levels ($\alpha\in\{0.3,0.1,0.05\}$). TRISHUL preserves
substantially higher subspace compatibility; at $\alpha=0.05$, TRISHUL's
alignment is approximately $0.35$ versus $\approx 0.15$ for RAVAN and
near zero for TRISHUL without shrinkage. The right panel shows spectral
entropy of client updates under increasing non-IID severity. The entropy
gap widens as $\alpha$ decreases, indicating that nuclear-norm shrinkage remains effective and increasingly important under severe client heterogeneity. Together, both metrics show that TRISHUL maintains stronger
spectral control as client distributions diverge.}
\label{fig:alpha_diagnostics}
\end{figure}

\begin{figure}
    \centering
    \includegraphics[width=\columnwidth]{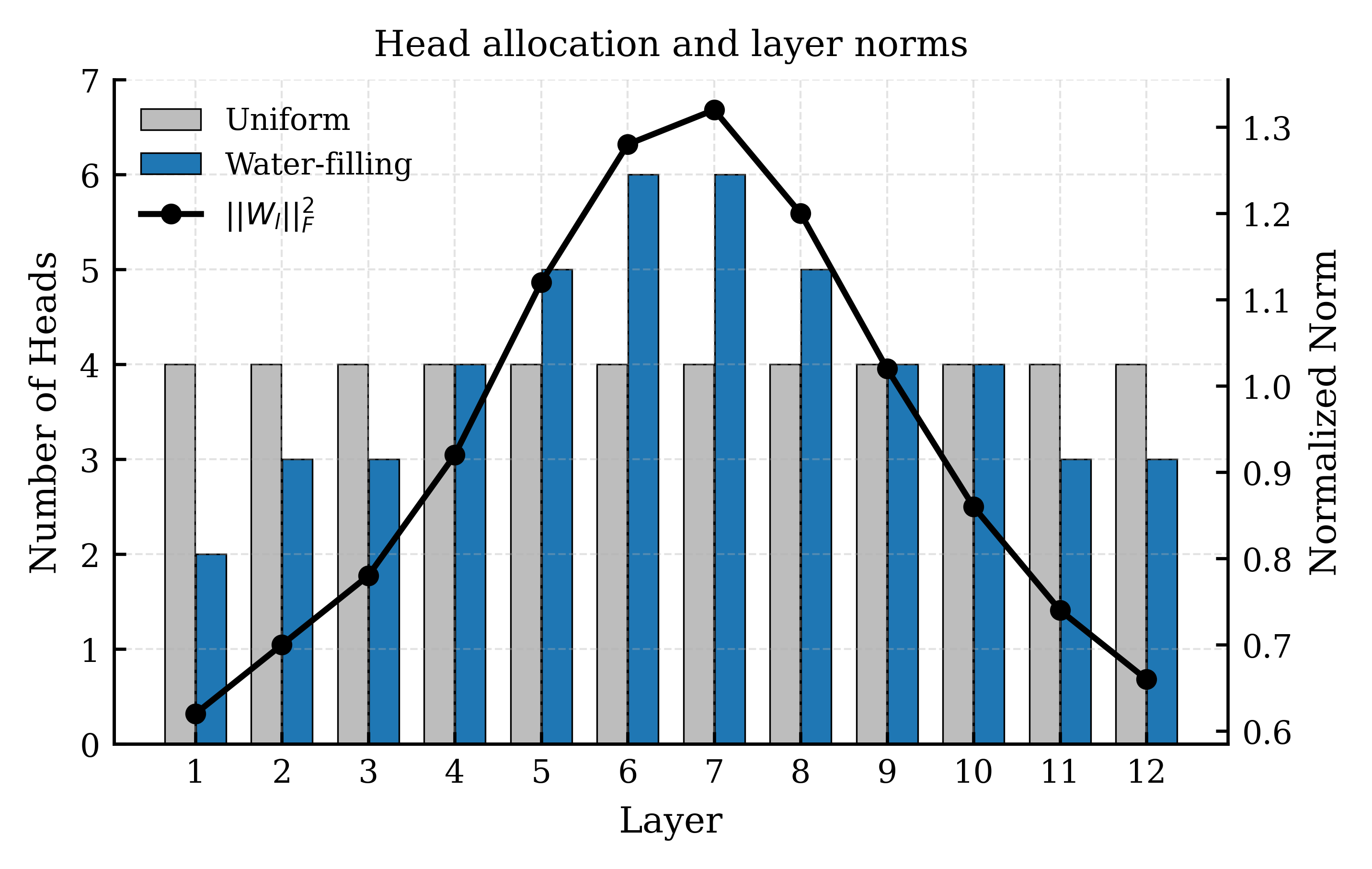}
    \caption{Head allocation per layer in ViT-B/16 (12 layers, total budget $h=48$) under uniform assignment (4 heads per layer) and concave water-filling (Eq.~\eqref{eq:waterfill_sol}). Layers with larger pretrained Frobenius norms $\|\mathbf{W}_l\|_F^2$ (shown on the secondary axis) receive more heads under TRISHUL. The allocation is computed once before training from the pretrained model and incurs no runtime overhead.}
\label{fig:waterfilling}
\end{figure}

\begin{table}
\caption{Ablation: head allocation strategies (lower budget, 20 clients).}
\label{tab:ablation_waterfilling}
\centering
\resizebox{\columnwidth}{!}{%
\begin{tabular}{l cc cc cc}
\toprule
\multirow{2}{*}{Allocation Strategy} 
& \multicolumn{2}{c}{CIFAR-100} 
& \multicolumn{2}{c}{SVHN} 
& \multicolumn{2}{c}{20 Newsgroups} \\
\cmidrule(lr){2-3} \cmidrule(lr){4-5} \cmidrule(lr){6-7}
& IID & Non-IID & IID & Non-IID & IID & Non-IID \\
\midrule
Uniform               & 84.41 & 77.82 & 94.23 & 90.88 & 68.96 & 66.91 \\
Gradient-norm-based   & 84.63 & 78.14 & 94.38 & 91.02 & 69.14 & 67.23 \\
Loss-curvature-based  & 84.57 & 78.41 & 94.41 & 91.17 & 69.27 & 67.58 \\
Concave water-filling (TRISHUL) & \textbf{85.13} & \textbf{79.41} & \textbf{94.72} & \textbf{91.38} & \textbf{69.82} & \textbf{68.14} \\
\bottomrule
\end{tabular}}
\end{table}

\paragraph{Layer-wise allocation behavior}
Figures~\ref{fig:waterfilling} and~\ref{fig:waterfilling_layerwise} visualize
the concave water-filling allocation against uniform assignment in ViT-B/16 (12 transformer layers, total budget $h=48$). Uniform allocation assigns exactly 4 heads to every layer. Concave water-filling assigns between 2 and 6 heads depending on the pretrained capacity scores: layers with larger pretrained norms (e.g. layers 5-8 in ViT-B/16, which correspond to mid-network attention blocks with stronger representational capacity) receive up to 6 heads, while early and late layers with weaker norms receive as few as 2--3. This non-uniform distribution concentrates trainable capacity where it can be most productively absorbed, directly improving the total representational gain $\sum_l h_l^*\|\mathbf{W}_l\|_F^2$ over the uniform baseline (Eq.~\eqref{eq:waterfill_dominates}). Critically, this allocation is computed once from pretrained norms before training begins and requires no runtime statistics, making it both principled and practically free for edge deployment.

\begin{figure}[!t]
\centering
\includegraphics[width=\linewidth]{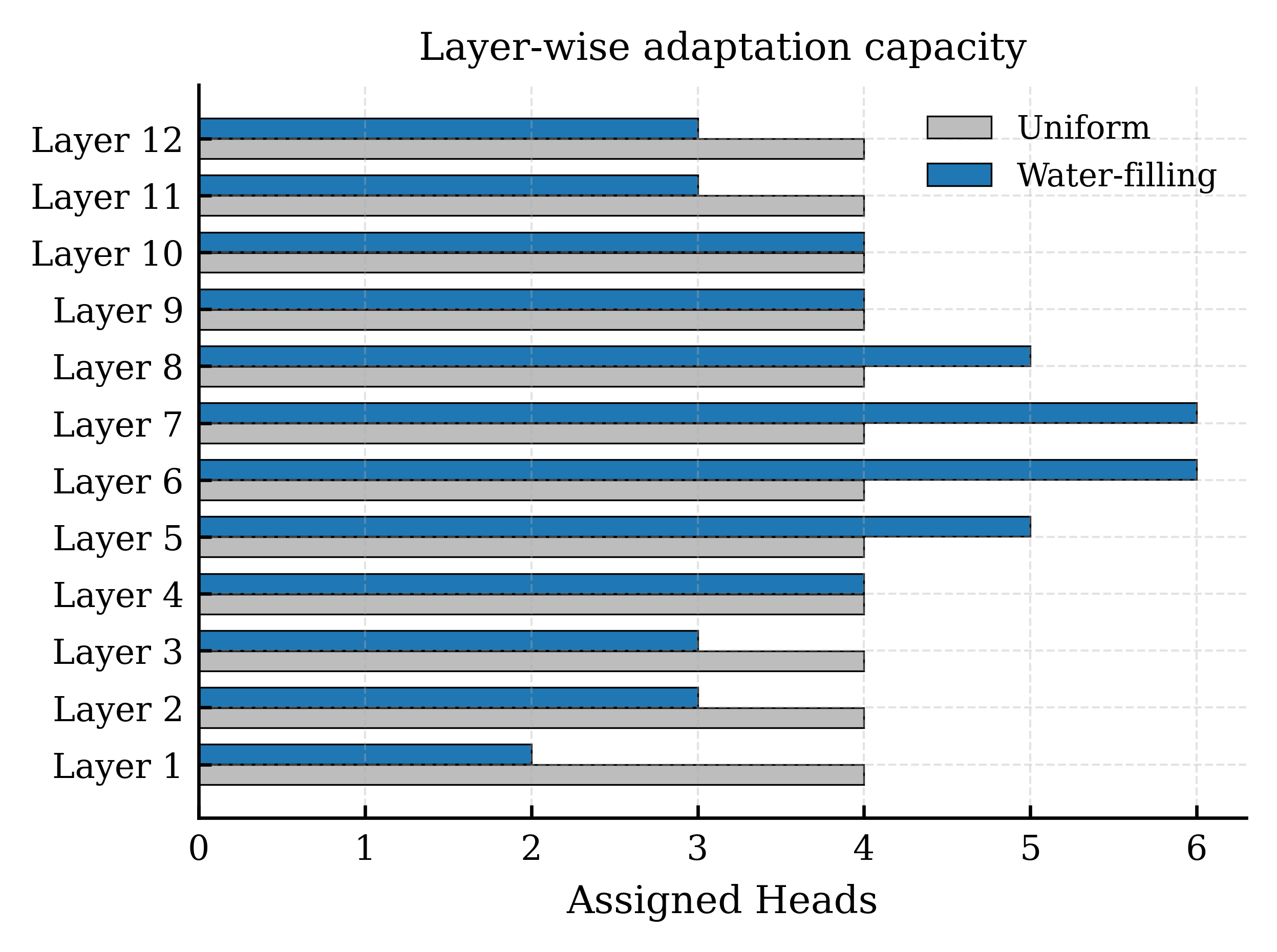}
\caption{Layer-wise head allocation in ViT-B/16 under uniform assignment and concave water-filling. Mid-network layers (roughly layers 5--8) with stronger pretrained representational capacity receive up to 6 heads, while shallower and deeper layers with weaker norms receive as few as 2--3. %Concave water-filling maximizes the diminishing-return allocation utility in Eq.~\eqref{eq:waterfill_opt} and improves that utility over uniform assignment when layer capacities differ.
}
\label{fig:waterfilling_layerwise}
\end{figure}

\begin{table}
\caption{Ablation: fixed vs. trainable scaling factors (lower budget, 20 clients).}
\label{tab:ablation_scaling}
\centering
\resizebox{\columnwidth}{!}{%
\begin{tabular}{l cc cc cc}
\toprule
\multirow{2}{*}{Scaling} 
& \multicolumn{2}{c}{CIFAR-100} 
& \multicolumn{2}{c}{SVHN} 
& \multicolumn{2}{c}{20 Newsgroups} \\
\cmidrule(lr){2-3} \cmidrule(lr){4-5} \cmidrule(lr){6-7}
& IID & Non-IID & IID & Non-IID & IID & Non-IID \\
\midrule
Fixed (\(s_i=1\)) & 84.72 & 77.62 & 94.41 & 90.85 & 69.21 & 66.98 \\
Trainable (TRISHUL) & \textbf{85.13} & \textbf{79.41} & \textbf{94.72} & \textbf{91.38} & \textbf{69.82} & \textbf{68.14} \\
\bottomrule
\end{tabular}}
\end{table}

\begin{table}
\caption{Ablation: initialization strategies (lower budget, 20 non-IID clients).}
\label{tab:ablation_init}
\centering
\resizebox{0.75\columnwidth}{!}{%
\begin{tabular}{l ccc}
\toprule
Initialization  & CIFAR-100 & SVHN & 20 Newsgroups \\
\midrule
Random Normal   & 78.21 & 90.15 & \textbf{68.14} \\
Gram-Schmidt    & \textbf{79.41} & \textbf{91.38} & 66.92 \\
Constant        & 58.30 & 88.12 & 56.85 \\
Shared Subspace & 57.51 & 84.67 & 55.73 \\ 
\bottomrule
\end{tabular}}
\end{table}
\begin{table}
\caption{Accuracy under computational heterogeneity (CIFAR-100, 20 non-IID clients, lower budget).}
\label{tab:ablation_heterogeneity}
\centering
\resizebox{0.85\columnwidth}{!}{%
\begin{tabular}{l ccc}
\toprule
Method & Bell-shaped & Uniform & Skewed-right \\
\midrule
HetLoRA                 & 75.23 & 73.41 & 68.92 \\
FlexLoRA                & 76.11 & 74.28 & 70.14 \\
TRISHUL (weight-based)  & 77.84 & 76.32 & 72.45 \\
TRISHUL (gradient-based)& 78.21 & 77.03 & 73.68 \\
TRISHUL (random)        & \textbf{79.06} & \textbf{78.14} & \textbf{75.31} \\
\bottomrule
\end{tabular}}
\end{table}

\paragraph{Client update covariance}
We also visualize the covariance structure of client updates in Figure~\ref{fig:covariance}. 
The left panel shows RAVAN, where the covariance matrix has a noisy structure with substantial off-diagonal variation, indicating that client updates are highly client-specific and poorly aligned, exactly the regime in which aggregation variance is large. 
The right panel shows TRISHUL, where the covariance matrix is substantially cleaner: the diagonal structure is more pronounced and off-diagonal entries are suppressed toward zero, indicating that updates across clients become more coherent after nuclear-norm spectral shrinkage. 
This structural difference directly corresponds to the reduced aggregation variance and lower spectral entropy observed in Figure~\ref{fig:spectral_diagnostics}, providing a holistic view of how TRISHUL improves cross-client update consistency.

\begin{figure}[!t]
\centering
\begin{minipage}[t]{0.48\linewidth}
    \centering
    \includegraphics[width=\linewidth]{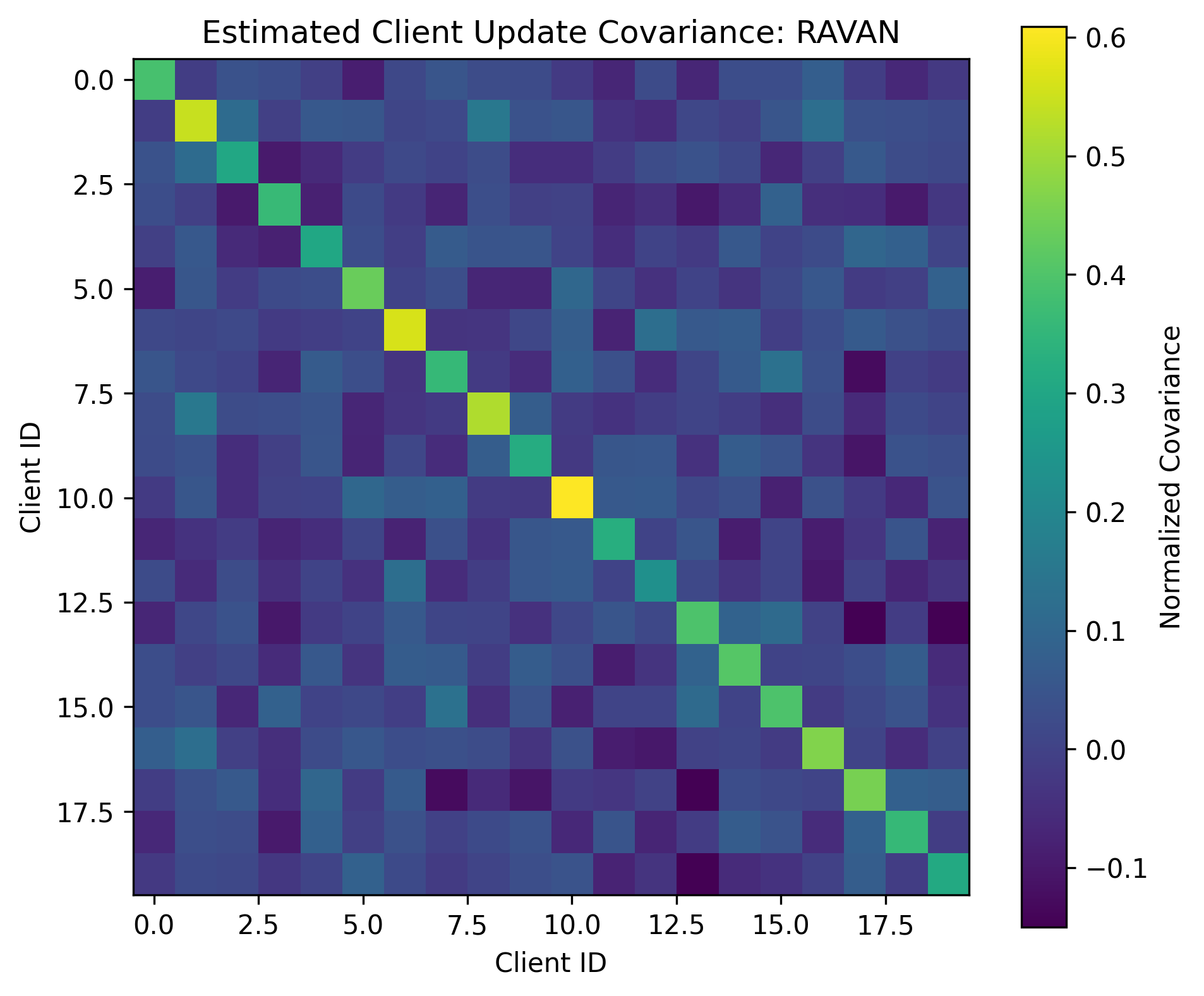}
    \small{(a) RAVAN}
\end{minipage}
\hfill
\begin{minipage}[t]{0.48\linewidth}
    \centering
    \includegraphics[width=\linewidth]{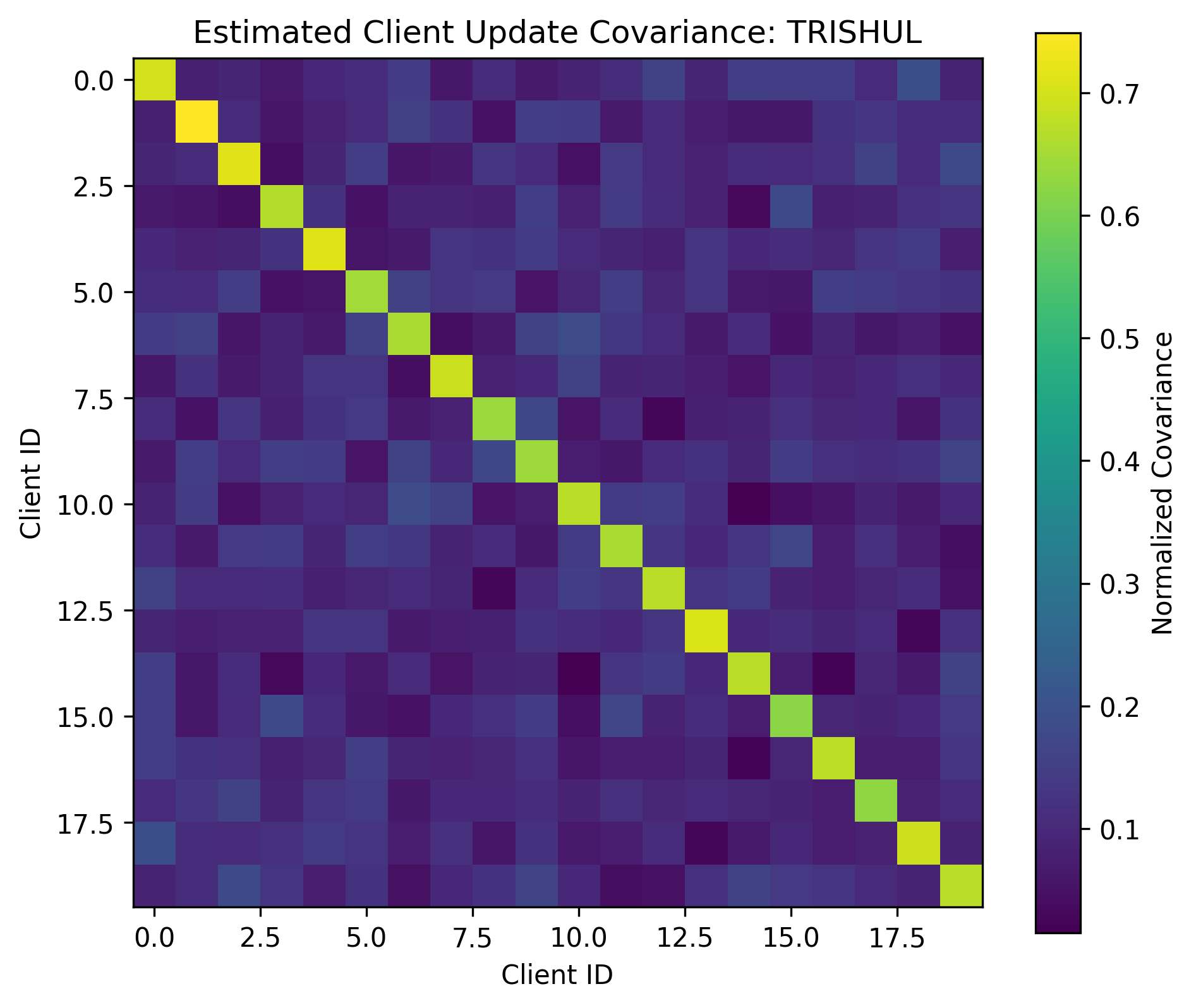}
    \small{(b) TRISHUL}
\end{minipage}
\caption{Client update covariance matrices under non-IID CIFAR-100 ($\alpha=0.3$, 20 clients, lower budget). \textbf{(a)} RAVAN exhibits a noisy structure with high off-diagonal variation, indicating strongly client-specific updates with little shared direction across clients. \textbf{(b)} TRISHUL produces a substantially cleaner structure with a more pronounced diagonal and suppressed off-diagonal entries, reflecting improved cross-client update coherence after nuclear-norm spectral shrinkage. This structural improvement directly reduces aggregation variance and improves global model quality.}
\label{fig:covariance}
\end{figure}

\subsubsection{Influence of Concave Water-Filling Head Allocation}
\label{subsubsec:waterfilling_ablation}
We compare the proposed concave water-filling allocation against three alternatives: uniform allocation, gradient-norm-based allocation, and loss-curvature-based allocation. As shown in Table~\ref{tab:ablation_waterfilling}, concave water-filling performs best across all datasets. The two heuristic alternatives improve on uniform assignment, which supports the general importance of adaptive capacity allocation, but both require additional statistics to be estimated during training. Concave water-filling, by contrast, is computed once from pretrained layer norms and incurs essentially zero runtime overhead. This makes it both more principled and more practical for edge deployment.

\subsubsection{Role of Trainable Scaling Factors}
\label{subsubsec:scaling_factors}

Trainable head-wise scaling factors allow clients to modulate the relative importance of different adaptation heads. Table~\ref{tab:ablation_scaling} shows that learning these scalars consistently improves performance, especially under non-IID data. The improvements suggest that local tasks do not use all adaptation heads equally and that lightweight head-wise modulation helps each client exploit the fixed adaptation budget more effectively.

\subsubsection{Effect of Initialization}
\label{subsubsec:ablation_init}

Because the basis matrices remain frozen during training, their initialization influences the expressive geometry of the entire multi-head adaptation space. Table~\ref{tab:ablation_init} shows that Gram--Schmidt initialization is best for vision tasks, whereas random-normal initialization performs best on 20~Newsgroups. Shared-subspace and constant initializations perform substantially worse, suggesting that inter-head subspace diversity is important for effective multi-head adaptation.

\subsubsection{Computational Heterogeneity and Head Selection}
\label{subsubsec:comp_heterogeneity}

Finally, we study computational heterogeneity by assigning clients different trainable budgets drawn from bell-shaped, uniform, and skewed-right distributions. Clients with lower budgets update only a subset of heads. Table~\ref{tab:ablation_heterogeneity} compares three head-selection strategies: random, weight-based, and gradient-based. Random selection performs best overall. Under non-IID data, weight-based and gradient-based selection tend to prefer locally dominant heads, which can reflect overfitting to client-specific distributions. Random selection avoids this bias and yields more globally representative updates. TRISHUL with random selection outperforms both HetLoRA and FlexLoRA across all heterogeneity profiles, indicating that partial-head freezing is an effective mechanism for handling device heterogeneity.

% \subsection{Limitations}
% \label{subsec:limitations}
% TRISHUL assumes that the shared frozen bases span useful adaptation subspaces for the downstream task. If the basis initialization is poorly aligned with task-relevant directions, the method may preserve compact but suboptimal update modes. The water-filling allocation is computed once from pretrained layer capacity and does not adapt to task-specific sensitivity during training. TRISHUL also follows the standard FL raw-data non-sharing assumption but does not provide formal privacy guarantees by itself; secure aggregation or differential privacy is still required when update leakage is a concern. Finally, the present evaluation is simulation-based and should be extended to real edge deployments with wireless variability, stragglers, energy constraints, intermittent participation, and quantized communication.

\section{Conclusion and Future Work}
\label{sec:conclusion}

We developed {TRISHUL} as a three-pronged spectral control framework for heterogeneous federated PEFT. TRISHUL addresses two central failure modes of federated PEFT: unstable client update spectra that cause destructive aggregation interference, and naive uniform allocation of limited adaptation capacity across layers. It does so through three complementary components: exact multi-head low-rank aggregation, nuclear-norm proximal spectral shrinkage, and concave water-filling head allocation.

The experimental results show that these components work together effectively. Across vision tasks, language tasks, and billion-parameter model fine-tuning, TRISHUL consistently outperforms strong federated PEFT baselines. The gains are largest under stronger heterogeneity, which supports the main claim of the paper that spectral inconsistency is a fundamental source of federated PEFT degradation and must be controlled explicitly. At the same time, the method preserves practical efficiency: it introduces no additional communication overhead and adds less than $1\%$ computational overhead per round.

Importantly, TRISHUL is a spectral-control framework that filters client-specific singular directions before aggregation, reducing effective rank, spectral entropy, and inter-client variance while preserving PEFT-level communication efficiency. Several directions remain open. First, the current head allocation is fixed before training; dynamic allocation strategies that adapt to evolving layer importance may yield further gains. Second, all heads currently use the same rank $r$; heterogeneous-rank variants could improve parameter efficiency. Third, extending TRISHUL to differentially private federated learning is promising, since spectrally compact updates may admit tighter privacy calibration. Finally, adaptive selection of the regularization strength $\lambda$ and robustness under low-precision or quantized training merit further investigation. Overall, the results suggest that explicit spectral control is a powerful and practical principle for federated fine-tuning of large models at the edge.

\bibliographystyle{IEEEtran}
\bibliography{references}
\newpage
\appendix

\section{Theory of TRISHUL}
\label{app:theory}

This appendix formalizes the three mechanisms used by TRISHUL: exact shared-basis aggregation, rank-controlled spectral shrinkage, and concave water-filling allocation. The results are intended to characterize the mechanism; they do not claim that spectral shrinkage alone solves all statistical heterogeneity in FL.

\subsection{Preliminaries and Assumptions}

We consider the federated objective
\begin{equation}
    F(\mathbf{W})
    =
    \sum_{c=1}^{C}p_c\mathcal{L}_c(\mathbf{W}),
    \qquad
    p_c=\frac{n_c}{\sum_j n_j}.
    \label{eq:appendix_fl_obj}
\end{equation}
For a fixed layer and head, define the uploaded core update
\begin{equation}
    \mathbf{Z}_{c,i}
    =
    s_{c,i}\widetilde{\mathbf{H}}_{c,i},
    \label{eq:uploaded_core}
\end{equation}
where $\widetilde{\mathbf{H}}_{c,i}$ is the post-SVT core.

\begin{assumption}[Smooth local objectives]
Each local objective $\mathcal{L}_c$ is $L$-smooth.
\end{assumption}

\begin{assumption}[Bounded post-SVT cores]
For every client $c$ and head $i$, the post-SVT core satisfies
\[
    \|\widetilde{\mathbf{H}}_{c,i}\|_2\leq M_i,
    \qquad
    |s_{c,i}|\leq s_{\max}.
\]
\end{assumption}

\begin{assumption}[Unbiased client sampling]
Participating clients are sampled independently according to the aggregation weights used by the server.
\end{assumption}

\subsection{Exactness of Shared-Basis Aggregation}

\begin{theorem}[Exact weighted aggregation]
\label{thm:exact}
Let $\mathbf{B}_{l,i}$ and $\mathbf{A}_{l,i}$ be frozen and shared across clients. For any weights $\{\omega_c\}$ satisfying $\sum_c\omega_c=1$, the aggregated ambient update satisfies
\begin{equation}
    \sum_c \omega_c
    \sum_i
    \mathbf{B}_{l,i}
    (s_{c,l,i}\mathbf{H}_{c,l,i})
    \mathbf{A}_{l,i}
    =
    \sum_i
    \mathbf{B}_{l,i}
    \left(
    \sum_c \omega_c s_{c,l,i}\mathbf{H}_{c,l,i}
    \right)
    \mathbf{A}_{l,i}.
    \label{eq:appendix_exact}
\end{equation}
\end{theorem}

\begin{proof}
The result follows by linearity of summation and by the fact that the bases are fixed with respect to the client index:
\[
\sum_c\omega_c\sum_i \mathbf{B}_{l,i}\mathbf{Z}_{c,l,i}\mathbf{A}_{l,i}
=
\sum_i\mathbf{B}_{l,i}\left(\sum_c\omega_c\mathbf{Z}_{c,l,i}\right)\mathbf{A}_{l,i}.
\]
\end{proof}

\begin{remark}
Theorem~\ref{thm:exact} removes the algebraic factorization bias of standard LoRA aggregation, where generally $\mathbb{E}[\mathbf{B}_c\mathbf{A}_c]\neq \mathbb{E}[\mathbf{B}_c]\mathbb{E}[\mathbf{A}_c]$. It does not imply that the aggregated update is statistically optimal under arbitrary non-IID data; spectral shrinkage is introduced to reduce the remaining client-update variance.
\end{remark}

\subsection{Rank-Controlled Variance Reduction}

\begin{theorem}[Rank-controlled aggregation variance]
\label{thm:variance}
Let $\widetilde{\mathbf{H}}_{c,i}$ be obtained by SVT:
\begin{equation}
    \widetilde{\mathbf{H}}_{c,i}
    =
    \operatorname{prox}_{\lambda\eta\|\cdot\|_*}
    (\mathbf{H}_{c,i}^{\mathrm{grad}}).
\end{equation}
Suppose $\operatorname{rank}(\widetilde{\mathbf{H}}_{c,i})\leq \rho_i$, $\|\widetilde{\mathbf{H}}_{c,i}\|_2\leq M_i$, and $|s_{c,i}|\leq s_{\max}$. For $m$ independently sampled clients,
\begin{equation}
    \mathbb{E}
    \left\|
    \frac{1}{m}\sum_{c=1}^{m}
    \left(
    s_{c,i}\widetilde{\mathbf{H}}_{c,i}
    -
    \mathbb{E}[s_{c,i}\widetilde{\mathbf{H}}_{c,i}]
    \right)
    \right\|_F^2
    \leq
    \frac{s_{\max}^2\rho_i M_i^2}{m}.
    \label{eq:appendix_variance}
\end{equation}
\end{theorem}

\begin{proof}
For any matrix $\mathbf{X}$ with rank at most $\rho_i$,
\[
    \|\mathbf{X}\|_F^2
    =
    \sum_{j=1}^{\rho_i}\sigma_j^2(\mathbf{X})
    \leq
    \rho_i\|\mathbf{X}\|_2^2.
\]
Thus $\|s_{c,i}\widetilde{\mathbf{H}}_{c,i}\|_F^2\leq s_{\max}^2\rho_iM_i^2$. The independence of sampled clients gives
\[
    \mathbb{E}\left\|
    \frac{1}{m}\sum_{c=1}^{m}(\mathbf{Z}_{c,i}-\mathbb{E}\mathbf{Z}_{c,i})
    \right\|_F^2
    =
    \frac{1}{m}
    \mathbb{E}\|\mathbf{Z}_{c,i}-\mathbb{E}\mathbf{Z}_{c,i}\|_F^2
\]
\[
    \leq
    \frac{1}{m}\mathbb{E}\|\mathbf{Z}_{c,i}\|_F^2,
\]
which proves the claim.
\end{proof}

\begin{remark}[Bias--variance tradeoff]
Increasing $\lambda$ increases the SVT threshold $\tau=\lambda\eta$, so fewer singular modes survive and $\rho_i$ decreases. This tightens the variance bound in~\eqref{eq:appendix_variance}, but it may also increase approximation bias by removing task-relevant modes.
\end{remark}

\subsection{Optimality of Concave Water-Filling Allocation}

\begin{theorem}[Water-filling solution]
\label{thm:waterfill}
Let $a_l=\|\mathbf{W}_l\|_F^2+\epsilon>0$. The solution of
\begin{align}
    \max_{\{h_l\geq0\}}
    \;&
    \sum_{l=1}^{L}a_l\log(1+h_l)
    \nonumber\\
    \mathrm{s.t.}
    \;&
    \sum_{l=1}^{L}h_l=h
    \label{eq:appendix_water_obj}
\end{align}
is
\begin{equation}
    h_l^*
    =
    \left[\frac{a_l}{\mu}-1\right]_+,
    \label{eq:appendix_water_solution}
\end{equation}
where $\mu>0$ is chosen such that $\sum_l h_l^*=h$.
\end{theorem}

\begin{proof}
The Lagrangian is
\[
    \mathcal{L}
    =
    \sum_l a_l\log(1+h_l)
    -
    \mu\left(\sum_l h_l-h\right)
    +
    \sum_l\nu_l h_l,
\]
with $\nu_l\geq0$. The KKT stationarity condition gives
\[
    \frac{a_l}{1+h_l}-\mu+\nu_l=0.
\]
For active layers $h_l>0$, $\nu_l=0$ and $h_l=a_l/\mu-1$. For inactive layers, complementary slackness gives $h_l=0$. Combining both cases yields~\eqref{eq:appendix_water_solution}.
\end{proof}

\begin{corollary}[Utility dominance over uniform allocation]
Let $h_l^{\mathrm{unif}}=h/L$. Since $\{h_l^*\}$ solves~\eqref{eq:appendix_water_obj},
\begin{equation}
    \sum_{l=1}^{L}a_l\log(1+h_l^*)
    \geq
    \sum_{l=1}^{L}a_l\log(1+h_l^{\mathrm{unif}}).
\end{equation}
The inequality is strict when the uniform allocation does not satisfy the KKT conditions.
\end{corollary}

\subsection{Convergence to a Stationary Point}
\label{subsec:convergence_appendix}

We now show that TRISHUL converges to a stationary point of the composite nuclear-norm regularized federated PEFT objective. Since the pretrained weights and bases $\{\mathbf{B}_{l,i},\mathbf{A}_{l,i}\}$ are frozen, the optimization variables are only the compact cores and head scalars. Let
\[
    \boldsymbol{\theta}
    =
    \operatorname{vec}
    \left(
    \{\mathbf{H}_{l,i}\}_{l,i},
    \{s_{l,i}\}_{l,i}
    \right)
\]
denote the concatenated trainable TRISHUL parameter vector. For client $c$, define
\[
    f_c(\boldsymbol{\theta})
    =
    \mathcal{L}_c
    \left(
    \mathbf{W}^0+\Delta\mathbf{W}(\boldsymbol{\theta})
    \right),
    \qquad
    F(\boldsymbol{\theta})
    =
    \sum_{c=1}^{C}p_c f_c(\boldsymbol{\theta}).
\]
The composite TRISHUL objective is
\begin{equation}
    \Phi(\boldsymbol{\theta})
    =
    F(\boldsymbol{\theta})
    +
    \Psi(\boldsymbol{\theta}),
    \label{eq:composite_objective}
\end{equation}
where
\begin{equation}
    \Psi(\boldsymbol{\theta})
    =
    \lambda
    \sum_{l=1}^{L}
    \sum_{i=1}^{h_l}
    \|\mathbf{H}_{l,i}\|_*
    +
    \iota_{\mathcal{S}}(\boldsymbol{\theta}),
    \label{eq:composite_regularizer}
\end{equation}
and $\iota_{\mathcal{S}}$ is the indicator of the scalar-feasible set
\[
    \mathcal{S}
    =
    \{\boldsymbol{\theta}:0\leq s_{l,i}\leq s_{\max},\ \forall l,i\}.
\]
Thus the nuclear-norm SVT step and scalar clipping are both represented by the proximal operator of $\Psi$.

Because $\Psi$ is nonsmooth, stationarity is measured by the proximal gradient mapping
\begin{equation}
    \mathcal{G}_{\eta}(\boldsymbol{\theta})
    =
    \frac{1}{\eta}
    \left[
    \boldsymbol{\theta}
    -
    \operatorname{prox}_{\eta\Psi}
    \left(
    \boldsymbol{\theta}-\eta\nabla F(\boldsymbol{\theta})
    \right)
    \right].
    \label{eq:proximal_gradient_mapping}
\end{equation}
A point $\boldsymbol{\theta}$ is stationary for $\Phi$ if
$\mathcal{G}_{\eta}(\boldsymbol{\theta})=\mathbf{0}$, equivalently
\[
    \mathbf{0}
    \in
    \nabla F(\boldsymbol{\theta})
    +
    \partial\Psi(\boldsymbol{\theta}).
\]

\begin{assumption}[Smoothness]
\label{assump:conv_smoothness}
Each client objective $f_c$ is $L$-smooth in the compact TRISHUL variables, i.e.,
\[
    \|\nabla f_c(\boldsymbol{\theta})-\nabla f_c(\boldsymbol{\theta}')\|
    \leq
    L\|\boldsymbol{\theta}-\boldsymbol{\theta}'\|
\]
for all $\boldsymbol{\theta},\boldsymbol{\theta}'$. Consequently, $F$ is also $L$-smooth.
\end{assumption}

\begin{assumption}[Unbiased stochastic gradients]
\label{assump:conv_unbiased}
At each local step, client $c$ computes a stochastic gradient $\mathbf{g}_c(\boldsymbol{\theta};\xi)$ satisfying
\[
    \mathbb{E}_{\xi}[\mathbf{g}_c(\boldsymbol{\theta};\xi)]
    =
    \nabla f_c(\boldsymbol{\theta}),
\]
and
\[
    \mathbb{E}_{\xi}
    \|\mathbf{g}_c(\boldsymbol{\theta};\xi)-\nabla f_c(\boldsymbol{\theta})\|^2
    \leq
    \sigma^2.
\]
\end{assumption}

\begin{assumption}[Bounded client heterogeneity]
\label{assump:conv_heterogeneity}
There exists $\zeta^2\geq0$ such that, for all $\boldsymbol{\theta}$,
\[
    \sum_{c=1}^{C}p_c
    \|\nabla f_c(\boldsymbol{\theta})-\nabla F(\boldsymbol{\theta})\|^2
    \leq
    \zeta^2.
\]
\end{assumption}

\begin{assumption}[Bounded second moment]
\label{assump:conv_second_moment}
There exists $G^2>0$ such that, for all clients and all local iterates,
\[
    \mathbb{E}\|\mathbf{g}_c(\boldsymbol{\theta};\xi)\|^2\leq G^2.
\]
\end{assumption}

\begin{assumption}[Client sampling and partial-head aggregation]
\label{assump:conv_sampling}
At each communication round, $m$ clients are sampled independently using probabilities consistent with the aggregation weights. For partial-head participation, let $m_{\mathrm{eff}}$ denote the minimum effective number of clients contributing to any active head in a round. Server aggregation is unbiased with respect to the sampled clients.
\end{assumption}

\begin{assumption}[Lower bounded objective]
\label{assump:conv_lower_bound}
The composite objective is bounded below:
\[
    \Phi^\star
    =
    \inf_{\boldsymbol{\theta}}\Phi(\boldsymbol{\theta})
    >
    -\infty.
\]
\end{assumption}

\begin{theorem}[Convergence to a stationary point]
\label{thm:stationary_convergence}
Suppose Assumptions~\ref{assump:conv_smoothness}--\ref{assump:conv_lower_bound} hold. Let TRISHUL run for $T$ communication rounds with $S$ local proximal steps per selected client and stepsize
\[
    0<\eta\leq \frac{1}{8LS}.
\]
Let $\boldsymbol{\theta}^{(t)}$ denote the global TRISHUL parameter vector at the beginning of round $t$, and let $\bar t$ be sampled uniformly from $\{0,1,\ldots,T-1\}$. Then
\begin{align}
    \mathbb{E}
    \left[
    \|\mathcal{G}_{\eta}(\boldsymbol{\theta}^{(\bar t)})\|^2
    \right]
    \leq
    &\frac{8\left(\Phi(\boldsymbol{\theta}^{(0)})-\Phi^\star\right)}{\eta S T}
    +
    \frac{16L\eta\sigma^2}{m_{\mathrm{eff}}}
    \nonumber\\
    &+
    64L^2\eta^2(S-1)(G^2+\zeta^2).
    \label{eq:stationary_bound}
\end{align}
Consequently, for fixed $S$ and $\eta=\Theta(1/\sqrt{ST})$,
\begin{align}
    \mathbb{E}
    \left[
    \|\mathcal{G}_{\eta}(\boldsymbol{\theta}^{(\bar t)})\|^2
    \right]
    =
    &\mathcal{O}\!\left(\frac{1}{\sqrt{ST}}\right)
    +
    \mathcal{O}\!\left(\frac{\sigma^2}{m_{\mathrm{eff}}\sqrt{ST}}\right)\nonumber\\
    &+
    \mathcal{O}\!\left(\frac{(S-1)(G^2+\zeta^2)}{ST}\right).
    \label{eq:stationary_rate}
\end{align}
TUnder bounded stochastic variance and bounded client heterogeneity, TRISHUL converges to a stationary point of the nuclear-norm regularized federated PEFT objective.
\end{theorem}

\begin{proof}
Let $\boldsymbol{\theta}_{c,\tau}^{(t)}$ be the local parameter vector of client $c$ after $\tau$ local proximal steps in round $t$, with $\boldsymbol{\theta}_{c,0}^{(t)}=\boldsymbol{\theta}^{(t)}$. A local TRISHUL update can be written as
\[
    \boldsymbol{\theta}_{c,\tau+1}^{(t)}
    =
    \operatorname{prox}_{\eta\Psi}
    \left(
    \boldsymbol{\theta}_{c,\tau}^{(t)}
    -
    \eta\mathbf{g}_c(\boldsymbol{\theta}_{c,\tau}^{(t)};\xi_{c,\tau})
    \right).
\]
The proximal operator contains both SVT on the core matrices and clipping of the scalars.

For an $L$-smooth objective and a proper closed convex nonsmooth term $\Psi$, the standard proximal descent inequality gives
\begin{align}
    \mathbb{E}_{\xi}
    [\Phi(\boldsymbol{\theta}_{c,\tau+1}^{(t)})]
    \leq
    &\Phi(\boldsymbol{\theta}_{c,\tau}^{(t)})
    -
    \frac{\eta}{4}
    \mathbb{E}_{\xi}
    \|\mathcal{G}_{\eta}(\boldsymbol{\theta}_{c,\tau}^{(t)})\|^2
    \nonumber\\
    &+
    2L\eta^2
    \mathbb{E}_{\xi}
    \|\mathbf{g}_c(\boldsymbol{\theta}_{c,\tau}^{(t)};\xi_{c,\tau})
    -
    \nabla F(\boldsymbol{\theta}_{c,\tau}^{(t)})\|^2.
    \label{eq:prox_descent_step}
\end{align}
The stochastic-gradient error decomposes into stochastic noise and client heterogeneity:
\[
    \mathbb{E}
    \|\mathbf{g}_c-\nabla F\|^2
    \leq
    2\sigma^2+2\zeta^2,
\]
after averaging over sampled clients. Multiple local steps introduce drift between $\boldsymbol{\theta}_{c,\tau}^{(t)}$ and $\boldsymbol{\theta}^{(t)}$. Since the proximal operator is nonexpansive and the stochastic gradients have bounded second moment,
\[
    \mathbb{E}
    \|\boldsymbol{\theta}_{c,\tau}^{(t)}-\boldsymbol{\theta}^{(t)}\|^2
    \leq
    4\eta^2\tau G^2.
\]
By $L$-smoothness, this yields an additional local-drift term bounded by $16L^2\eta^2\tau(G^2+\zeta^2)$ in the gradient-mapping comparison.

After $S$ local steps, the server aggregates uploaded core products. Because all bases are frozen and shared, aggregation is linear in the compact core space and exactly corresponds to weighted aggregation of the ambient PEFT updates. Hence no factorization error appears in the descent analysis. The only aggregation noise comes from stochastic gradients, client sampling, and partial-head participation, summarized by $m_{\mathrm{eff}}$.

Combining the local descent inequalities over $S$ steps, averaging over participating clients, and using $\eta\leq1/(8LS)$ gives
\begin{align}
    \mathbb{E}[\Phi(\boldsymbol{\theta}^{(t+1)})]
    \leq
    &\mathbb{E}[\Phi(\boldsymbol{\theta}^{(t)})]
    -
    \frac{\eta S}{8}
    \mathbb{E}\|\mathcal{G}_{\eta}(\boldsymbol{\theta}^{(t)})\|^2
    \nonumber\\
    &+
    2L\eta^2S\frac{\sigma^2}{m_{\mathrm{eff}}}
    +
    8L^2\eta^3S(S-1)(G^2+\zeta^2).
    \label{eq:one_round_descent}
\end{align}
Rearranging, summing over $t=0,\ldots,T-1$, and using $\Phi(\boldsymbol{\theta}^{(T)})\geq\Phi^\star$ yields Eq.~\eqref{eq:stationary_bound}. Since $\bar t$ is uniformly sampled from $\{0,\ldots,T-1\}$, the averaged bound is exactly the desired bound on $\mathbb{E}\|\mathcal{G}_{\eta}(\boldsymbol{\theta}^{(\bar t)})\|^2$. Substituting $\eta=\Theta(1/\sqrt{ST})$ gives Eq.~\eqref{eq:stationary_rate}.
\end{proof}

\begin{corollary}[Stationarity of the unregularized federated loss]
\label{cor:unregularized_stationarity}
Let $Q=\sum_{l=1}^{L}h_l$ be the total number of TRISHUL heads and assume each core matrix has size $r\times r$. If
\[
    \mathbb{E}
    \|\mathcal{G}_{\eta}(\boldsymbol{\theta}^{(\bar t)})\|^2
    \leq
    \varepsilon^2,
\]
then the distance to stationarity of the original smooth federated objective $F$ satisfies
\begin{equation}
    \mathbb{E}
    \left[
    \operatorname{dist}^2
    \left(
    \mathbf{0},
    \nabla F(\boldsymbol{\theta}^{(\bar t)})
    \right)
    \right]
    \leq
    2\varepsilon^2+2\lambda^2Qr.
    \label{eq:unregularized_stationarity}
\end{equation}
\end{corollary}

\begin{proof}
For the composite objective, stationarity implies
$\mathbf{0}\in\nabla F(\boldsymbol{\theta})+\partial\Psi(\boldsymbol{\theta})$. For each $r\times r$ core matrix, any subgradient of the nuclear norm has spectral norm at most one and Frobenius norm at most $\sqrt{r}$. Since there are $Q$ core matrices,
\[
    \sup_{\mathbf{Z}\in\partial R(\boldsymbol{\theta})}\|\mathbf{Z}\|_F
    \leq
    \sqrt{Qr}.
\]
Therefore an $\varepsilon$-stationary point of the composite objective satisfies
\[
    \operatorname{dist}(\mathbf{0},\nabla F(\boldsymbol{\theta}))
    \leq
    \varepsilon+\lambda\sqrt{Qr}.
\]
Squaring and using $(a+b)^2\leq2a^2+2b^2$ proves Eq.~\eqref{eq:unregularized_stationarity}.
\end{proof}

Theorem~\ref{thm:stationary_convergence} shows that TRISHUL converges to a stationary point of the nuclear-norm regularized federated PEFT objective. The first term in Eq.~\eqref{eq:stationary_bound} is the optimization error and decreases with the total number of local updates $ST$. The second term captures stochastic-gradient and client-sampling noise, and decreases with the effective number of participating clients per head. The third term captures local drift caused by multiple local steps under statistical heterogeneity. Corollary~\ref{cor:unregularized_stationarity} further shows that stationarity for the original unregularized federated loss is recovered up to a controlled shrinkage bias of order $\lambda^2Qr$.

\end{document}